\documentclass[twoside,leqno,twocolumn]{article}
\usepackage{ltexpprt}
\usepackage{cuted}
\usepackage{balance}
\pdfoutput=1
\usepackage{twoopt}
\usepackage[numbers]{natbib}

\usepackage{fullpage}
\usepackage{float}

\usepackage{times}
\usepackage{latexsym}
\usepackage{amsmath}
\usepackage{amssymb}
\usepackage{amsfonts}
\usepackage{mathrsfs}
\usepackage{aliascnt}
\usepackage{graphicx}
\usepackage[dvipsnames]{xcolor}
\usepackage{color}
\usepackage{enumerate}
\usepackage{array}
\usepackage{enumitem}

\usepackage[hypertexnames=false,colorlinks=true,pdfpagemode=Usenone,urlcolor=Blue,linkcolor=RoyalBlue,citecolor=OliveGreen,pdfstartview=FitH]{hyperref}
\usepackage{algorithmic}
\usepackage{algorithm}

\newtheorem{claim}{Claim}

\newtheorem{remark}{Remark}
\newcommand{\qed}{\mbox{\ \ \ }\rule{6pt}{7pt} \bigskip}

\makeatletter
\newenvironment{oneshot}[1]{\@begintheorem{#1}{\unskip}}{\@endtheorem}
\makeatother

\newcommand{\AutoAdjust}[3]{\mathchoice{ \left #1 #2  \right #3}{#1 #2 #3}{#1 #2 #3}{#1 #2 #3} }
\newcommand{\Xcomment}[1]{{}}

\newcommand{\InParentheses}[1]{\AutoAdjust{(}{#1}{)}}
\newcommand{\InBrackets}[1]{\AutoAdjust{[}{#1}{]}}
\newcommand{\Ex}[2][]{\operatorname{\mathbf E}_{#1}\InBrackets{#2}}
\newcommand{\Exlong}[2][]{\operatornamewithlimits{\mathbf E}\limits_{#1}\InBrackets{#2}}
\newcommand{\Prx}[2][]{\operatorname{\mathbf{P}}_{#1}\InBrackets{#2}}

\newcommand{\twonorm}[1]{\left\| {#1} \right\|_2}

\newcommand{\vect}[1]{\ensuremath{\mathbf{#1}}}

\newcommand{\Z}{\mathbb{Z}}

\newcommand{\be}{\begin{equation}}
\newcommand{\ee}{\end{equation}}
\newcommand{\bee}{\begin{equation*}}
\newcommand{\eee}{\end{equation*}}

\newcommand{\argmax}{\mathop{\rm argmax}}

\newcommand{\gain}{g}

\newcommandtwoopt{\gainit}[2][i][t]{\gain_{#1#2}}
\newcommand{\gains}{\vect{\gain}}
\newcommand{\gainst}[1][t]{\gains_{#1}}
\newcommand{\gainsht}[1][t-1]{\gains_{[0,#1]}}

\newcommand{\Gain}{G}

\newcommandtwoopt{\Gainit}[2][i][t]{\Gain_{#1#2}}
\newcommand{\Gains}{\vect{\Gain}}
\newcommand{\Gainst}[1][t]{\Gains_{#1}}
\newcommand{\Gainsh}{\Gains_{[0,T]}}

\newcommand{\alg}{\mathcal{A}}
\newcommand{\algt}[1][t]{\alg_{#1}}

\newcommand{\dist}{\mathcal{D}}

\newcommand{\binomdist}[1][k]{\text{Binom}(#1,\frac{1}{2})}
\newcommand{\regret}{R}
\newcommand{\regretf}[1][T]{\regret_{_{#1}}}
\newcommand{\regretd}{\regret_{\delta}}

\newcommand{\vx}{\vect{x}}
\newcommand{\vp}{\vect{p}}

\newcommand{\srw}{\text{SRW}}

\begin{document}
\pagenumbering{gobble}
\title{Towards Optimal Algorithms for Prediction with Expert Advice}

\author{
Nick Gravin\thanks{Microsoft Research.
One Memorial Drive,
Cambridge, MA 02142. \texttt{ngravin@gmail.com}.}
\and
Yuval Peres\thanks{Microsoft Research.
One Microsoft Way, Redmond, WA 98052.
\texttt{peres@microsoft.com}, \texttt{balu2901@gmail.com}.
}
\and
Balasubramanian Sivan\footnotemark[2]
}

\date{}

\maketitle

\begin{abstract}
We study the classical problem of prediction with expert advice in the 
adversarial setting with a geometric stopping time. In 
1965,~\citeauthor{Cover65} gave the optimal algorithm for the case of 2 
experts. In this paper, we design the optimal algorithm, adversary and regret 
for the case of 3 experts. Further, we show that the optimal algorithm for $2$ 
and $3$ experts is a probability matching algorithm (analogous to Thompson 
sampling) against a particular randomized adversary. Remarkably, our proof 
shows that the probability matching algorithm is not only optimal against this 
particular randomized adversary, but also minimax optimal. 

Our analysis develops upper and lower bounds simultaneously, analogous to the 
primal-dual method. Our analysis of the optimal adversary goes through delicate 
asymptotics of the random walk of a particle between multiple walls. We use the 
connection we develop to random walks to derive an improved algorithm and 
regret bound for the case of $4$ experts, and, provide a general framework for 
designing the optimal algorithm and adversary for an arbitrary number of 
experts.

\end{abstract}

\section{Introduction}
\label{sec:intro}
Predicting future events based on past observations, a.k.a. prediction with expert advice, is a
classic problem in learning.
The experts framework was the first framework proposed for online learning and encompasses
several applications as special cases.
The underlying problem is an online optimization problem: a {\em player} has to make a decision at each time step, namely, decide which of the $k$ experts' advice to follow.
At every time $t$, an {\em adversary} sets gains for each expert: a gain of $\gainit$ for expert $i$ at time $t$. Simultaneously, the player, seeing the gains from all previous steps except $t$, has to choose an action, i.e., decide on which expert to follow. If the player follows expert $j(t)$ at time $t$, he gains $\gainit[j(t),][t]$. At the end of each step $t$, the gains associated with all experts are revealed to the player, and the player's choice is revealed to the adversary. In the {\em finite horizon model}, this process is repeated for $T$ steps, and the player's goal is to perform (achieve a cumulative gain) as close as possible to the best single action (best expert) in hindsight, i.e., to minimize his
\emph{regret} $\regretf$:
$$\regretf =
\max_{1\leq i \leq k} \sum_{t=1}^T \gainit - \sum_{t=1}^T \gainit[j(t),][t].$$
Apart from assuming that the $\gainit$'s are bounded in $[0,1]$, we don't assume anything else about the gains.
Just as natural as the finite horizon model is the model with a {\em geometric horizon}: the stopping time is a geometric random variable with expectation $\frac{1}{\delta}$. In other words, the process ends at any given step with probability $\delta$, independently of the past. 
In this paper, we study both the finite horizon model and the geometric horizon model.

\paragraph{Questions and motivation.} Given the breadth of applications and the 
significance of the experts problem in online learning, in this work we seek to 
understand and crisply characterize the structure of the optimal algorithm and 
the structure of the worst case input sequences. We ask: 
\begin{enumerate}
\item What is the precisely optimal algorithm and regret values?
\item Does the optimal algorithm have a succinct and intuitive description (even for $2$ experts)?
\item What are the hardest (adversarial) sequences of experts' gains and do they follow a succinct pattern?
\end{enumerate}

Our motivation in exploring these questions include the following. 
\begin{enumerate}
\item Half a century after Cover~\cite{Cover65} described the optimal adversary for the case of $2$ experts, we still do not have general insights about the structure of the optimal algorithm or the optimal adversarial sequences. 
\item Several applications of the experts paradigm involve dealing with a small constant number of experts. What amount of gain can the optimal algorithm get over the multiplicative weights algorithm for a constant number of experts?
\item The problem is theoretically clean and challenging: a priori it is not 
even clear if the algorithm and adversary are succinctly describable. It could 
well be that the optimal algorithm's actions depend on various aspects of 
history in a manner that cannot be succinctly described. 
\end{enumerate}

\paragraph{Notation:} We fix some notation before proceeding. 
We denote by $\Gainit$ the cumulative gain of expert $i$ after $t$ steps. Namely, $\Gainit=\sum_{s=1}^{t}\gainit[i][s]$. We show that in the worst-case instance there is no benefit in using gains other  than $0$ and $1$, so we restrict to $\gainit \in \{0,1\}$. The notion of optimality used for the experts framework is the minimax regret
obtained against all possible adversarial sequences of experts' predictions,
\emph{the adversary} for short. We study the optimal adversary that inflicts maximal
regret (maxmin) against all possible algorithms.

\subsection*{Our contributions}
{\em Balanced adversary:} Our first general insight about the structure of the optimal adversary is that it is {\em balanced across all experts} at every time step, i.e., irrespective of the experts' past gains, the adversary sets equal expected gain for each expert in this step. This insight is pervasive in this paper and greatly simplifies the problem in that it lets us describe the optimal (minimax/maxmin) regret without making any reference to the optimal algorithm. Indeed, {\em every} algorithm performs equally well against the maximin optimal adversary that equalizes the expected gains of all experts\footnote{We clarify that this doesn't mean all algorithms get the same minimax regret. It is only the maximin optimal adversary that is balanced and not every adversary. In other words, if the adversary is maximin optimal, all algorithms are equal. But if the algorithm is not minimax optimal, the optimal adversary for that algorithm is not maximin optimal and hence not necessarily balanced, and will inflict a larger regret on the said algorithm than a balanced adversary does.}, and gets a gain of the average over all the $k$ experts of the cumulative gain, namely, $\frac{1}{k}|\Gains|_{1}$. Given this, the adversary's problem of maximizing regret can be reduced to maximizing the difference between maximum and average of the cumulative gains vector, i.e., $|\Gains|_{\infty}-\frac{1}{k}|\Gains|_{1}$.
\\
\\
\noindent {\em Maximizing the number of collisions between the leading and second best expert:} Our second insight about the adversary's structure is that its objective, namely, $|\Gains|_{\infty}-\frac{1}{k}|\Gains|_{1}$ never changes in expectation, except when there is no unique expert with the largest cumulative gains. This is because, the maximim optimal adversary being balanced implies that all experts' cumulative gains increase equally in expectation, including the cumulative gain of the expert who is currently leading. Thus the two quantities $|\Gains|_{\infty}$ and $\frac{1}{k}|\Gains|_{1}$ increase equally in expectation implying that a balanced adversary will find no way to increase its objective in expectation. The only time when this breaks is when the leading expert is not unique: this is because in this situation, the probability that the maximum cumulative gain increases is the probability of the cumulative gain of any one of the leading experts getting increased. Given this, the adversary's problem essentially boils down to the crisp and challenging probability question of constructing gain sequences that maximizes the number of collisions between the leading and the second leading expert.  Namely, construct a balanced random walk in $\Z^k$ with the objective of maximizing the number of collisions between the largest and second largest coordinates. 
\\
\\
\noindent{\em Designing the random walk that maximizes the number of collisions:} We use this insight about the adversary's problem being a controlled random walk to construct such walks and hence succinct adversaries for the case of $k=2$ and $3$ experts (this also gives a simple alternative proof for Cover's optimal adversary). While the case of $k=2$ is special, the progress for $k=3$ crucially relies on the above mentioned insights. While constructing the optimal such random walk for general $k$ is still complicated, we believe that this reduction is powerful and gives a useful starting point for thinking about possible candidates that come close. For instance, the ``comb adversary'' described later in the introduction, has a simple but non-trivial structure that was inspired from the number-of-collisions characterization. 
\\
\\
\noindent{\em Probability matching algorithm:} We establish a strong connection between the structure of the optimal algorithm and the optimal adversary. Namely, the optimal algorithm is a probability matching algorithm (analogous to the popular Thompson sampling procedure) that follows each expert with the probability that  this expert finishes as the leader, when the sequence of gains is set by the optimal adversary .

\noindent We describe our results in detail below.
\paragraph{1. Two experts.} The optimal adversary, designed by~\citeauthor{Cover65}, chooses one expert
uniformly at random, sets a gain of $1$ for that expert, and a gain of $0$ for the other. We give a very simple characterization of the {\em unique} optimal algorithm in {\em both the finite horizon and the geometric horizon models}, namely, {\em follow each expert with the probability that he finishes as the leading expert\footnote{If there is a tie in the finite horizon model, we consider the leader to be the unique expert who did not have any expert ahead of him in the last two steps. In the geometric horizon model ties are just broken uniformly at random.} (the one with maximum cumulative gains), when gains are set by Cover's optimal adversary}. Further, this probability of finishing as a leading expert has a simple analytical expression in the geometric horizon model, given in Algorithm~\ref{alg:2opt} (see Theorem~\ref{thm:2opt}). The finite horizon has a simple expression too (see Theorem~\ref{thm:2pmatching}). The optimal algorithm achieves a regret of $\sqrt{\frac{T}{2\pi}}$ in the finite horizon model, and $\frac{1}{2}\frac{1}{\sqrt{2\delta}}$ in the geometric horizon model, respectively as $T \to \infty$ and $\delta \to 0$ (see Theorem~\ref{thm:2opt} for the precisely optimal regret for every $\delta$).

\begin{algorithm}
\caption{: Optimal Algorithm for Geometric Horizon Model with Two Experts}
\label{alg:2opt}
\begin{algorithmic}[1]
\STATE Initialize $\xi = \frac{1-\sqrt{1-(1-\delta)^2}}{1-\delta} \sim 1-\sqrt{2\delta}$ as $\delta \to 0$
\STATE Convention: Leading expert (larger cumulative gains) is numbered $1$, and lagging expert
is numbered $2$
\FOR {Each time step $t$ till the game stops}
\STATE Compute cumulative gains for both experts: $\Gainit[1][t] = \sum_{s=1}^{t} \gainit[1][s],\text{ and }
\Gainit[2][t] = \sum_{s=1}^{t} \gainit[2][s]$
\STATE Let $d = \Gainit[1][t] - \Gainit[2][t]$. Note that by definition $d \geq 0$
\STATE Follow the leading expert with probability he will finish as leader, namely, $p_1(d) =
1-\frac{1}{2}\xi^d$
\STATE Follow the lagging expert with probability he will finish as leader, namely, $p_2(d) = \frac{1}{2}\xi^d$
\ENDFOR
\end{algorithmic}
\end{algorithm}

\paragraph{2. Three experts.} We derive the precisely optimal algorithm, adversary and regret values for three experts in the geometric horizon model (see Theorem~\ref{thm:3opt}). The optimal regret as $\delta \to 0$ is asymptotic to $\frac{2}{3}\frac{1}{\sqrt{2\delta}}$.
The optimal adversary (as $\delta \to 0$)\footnote{The optimal adversary for all values of $\delta$ (asymptotic or not) is almost identical to this. We describe this in Section~\ref{sec:delta}.}  is as follows: it pairs up the middle and the lagging experts, and together this pair always disagrees with the leading expert. That is, the $\gainit$'s are of the form $(0,1,1)$ or $(1,0,0)$ where the ordering in the tuple captures the leading, middle and lagging experts (and do not refer to the identities of experts). The optimal algorithm is again a simple probability matching algorithm to the optimal adversary, that follows each expert with the probability that this expert finishes in the lead.

The case of $3$ experts is significantly more complicated than $2$ experts. In particular, while the optimal adversary for $2$ experts was discovered back in 1965 (\cite{Cover65}), the optimal adversary for $3$ experts was not known so far. 
The relative coordinate system we introduce, that numbers experts according to their cumulative gains, provides a convenient way to describe the optimal adversary.

\paragraph{3. Arbitrary number of experts.} All our basic results continue to 
hold in both geometric and finite horizon models. I.e., the optimal adversary 
(i) plays only $\gainit\in\{0,1\}$ (and not in $[0,1]$), (ii) is balanced, 
(iii) at every step plays only one of the finitely
many vertices of the convex polytope of balanced distributions. Prior to this work, given $T$ and $k$, an algorithm for computing the precisely optimal adversary was not known.  Result (iii) reduces the search space of the optimal adversary to finitely many balanced distributions, and thereby enables us to write a mundane dynamic program of size $O(T^k)$. Note that even after realizing that $\gainit \in \{0,1\}$ without loss of generality, the adversary has infinitely many balanced probability distributions available to choose from in every step, and thus, a priori it is not clear how to write a meaningful dynamic program.

Through this dynamic program, we found that the optimal adversary for $k \geq 4$ {\em does not} always have a simple description like for $k=2,3$. Further, we observe that unlike $k=2,3$, for $4$ experts, the optimal adversary is already $\delta$-dependent, and its actions at a given configuration of cumulative gains depend on the exact values of cumulative gains and not just their order. Nevertheless, inspired by results from this dynamic program, we conjecture that there is a simple adversary (``comb adversary'') which is asymptotically (in $\delta$ or $T$) optimal: split experts into two teams $\{1,3,5,\dots\}$ and $\{2,4,6,\dots\}$ and increment the gains of all experts in exactly one of these teams chosen uniformly at random.  We analyze the comb adversary for $k=4$ and show that as $\delta \to 0$, it inflicts a regret of $\frac{\pi}{4}\frac{1}{\sqrt{2\delta}}$ (see Theorem~\ref{thm:4brownian}). We observed that for reasonably small values of $\delta$, the optimal regret converges to our lower bound of $\frac{\pi}{4}\frac{1}{\sqrt{2\delta}}$ indicating that the comb adversary is indeed asymptotically optimal for $k=4$.

\paragraph{Remarks}
\begin{enumerate}[leftmargin=*]
\setlength\itemsep{0.2em}
\item 
In this work, we develop the optimal algorithm and adversary {\em simultaneously}, thereby, completely bridging the gap
between upper and lower bounds for a small number of experts (our analysis obtains the optimal regret for $k=2,3$ experts for every value of $\delta$, and not just asymptotically as $\delta \to 0$).

\item 
Although the optimal algorithm for $k=2,3$ experts does a probability matching with respect to a particular adversary it turns out that this algorithm is not only optimal against this adversary, but also minimax optimal against all possible adversaries.
Our algorithms for $k=2,3$ experts (also our conjectured optimal algorithm for $k$ experts) are simple and practical. One can implement our algorithms as follows: from any configuration of cumulative gains simulate the ``comb adversary'' till the end of the process and follow the expert who finishes in the lead. Simulating the comb adversary simply entails flipping a coin in every step and incrementing by $1$ the gains of the respective (odd or even numbered) team of experts.

\item The comb adversary that we introduce and analyze presents a simple to describe random process. But even for $k=3,4$ analyzing this process requires an understanding of non-trivial aspects of simple random walk (see Theorems~\ref{thm:3brownian} and~\ref{thm:4brownian}).  Developing a method to analyze this process for general $k$ is a clean and challenging question on random walks.
\end{enumerate}

\paragraph{Comparison with Multiplicative Weights Algorithm.} It is known that for general bounded gains the widely used multiplicative weights algorithm (MWA), obtains a $\sqrt{\frac{T\ln k}{2}}$ regret and this regret is asymptotically optimal as both $\{T,k\} \rightarrow \infty$ (see~\citet{BFHHSW97} and
its generalization by~\citet{HKW95}). However, asymptotic analysis in $k$ does not shed much light on the structure of the optimal algorithm and the hardest sequences of experts' gains: this is because the quantity $\sqrt{\frac{T\ln k}{2}}$ is insensitive if we employ $100$ times as many experts and it is not optimal for a small (constant) number of experts. In this paper, we show that the  optimal algorithm is not in the family of multiplicative weight algorithms. Namely, we show that the optimal algorithm cannot be expressed as a MWA or even as a convex combination of MWAs. We refer the reader to Appendix~\ref{app:mwa} for a detailed discussion and proof. 

\paragraph{Related work.}
In this work, our goal is to identify the structure of the optimal algorithm and the adversary via a precise and efficient algorithmic description. There are recent works that {\em characterize} the optimal algorithm/adversary and regret as the supremum or infimum of some stochastic process, rather than give an efficient algorithmic description. There is also a significant body of recent work that either identify {\em approximately} optimal algorithms/adversaries, or, identify special cases where the optimal adversary can be precisely described. All of these relaxations allow for solving more general frameworks. But this body of work doesn't identify the precisely optimal algorithm/adversary/regret for the classical setting. In contrast, in our work we show that probability matching is precisely optimal, and we identify the precisely optimal adversary for the classical setting.  We discuss all these lines of recent work after discussing some classics in this area.

As mentioned earlier, for the exact setting we consider in this work, the work of~\cite{Cover65} is most closely related as it gives the optimal adversary and algorithm for the case of $2$ experts.

\textit{Classic works:} The book by~\citet{BL06-book} is an excellent source for both applications and references. The prediction with experts advice paradigm was introduced by~\citet{LW94}
and~\citet{Vovk90}. The famous multiplicative weights update algorithm was
introduced independently by these two works: as the weighted majority algorithm
by~\citeauthor{LW94} and as the aggregating algorithm by~\citeauthor{Vovk90}.
The pioneering work of~\citet{BFHHSW97} considered $\{0,1\}$ outcome space for
nature and showed that for the absolute loss function $\ell(x,y) = |x-y|$ (or
$g(x,y) = 1-|x-y|$), the asymptotically optimal regret is $\sqrt{\frac{T\ln
k}{2}}$. This was later extended to $[0,1]$ outcomes for nature
by~\citet{HKW95}. The asymptotic optimality of $\sqrt{\frac{T\ln k}{2}}$ for
arbitrary loss (gain) functions follows from the analysis of~\citet{Bianchi99}.
When it is known beforehand that the cumulative loss of the optimal expert is
going to be small, the optimal regret can be considerably improved, and such
results were obtained by~\citet{LW94} and~\citet{FS97}. With certain
assumptions on the loss function, the simplest possible algorithm of following
the best expert already guarantees sub-linear regret~\citet{Hannan57}. Even
when the loss functions are unbounded, if the loss functions are exponential
concave, sub-linear regret can still be achieved~\citet{BK99}.

\textit{Recent works:}~\cite{LS14} consider a setting where the adversary is restricted to pick gain vectors from the basis vector space $\{\mathbf{e}_1,\dots,\mathbf{e}_k\}$. For this set of gain vectors, the only balanced adversary is to pick a random expert in every step. Since our analysis shows that the optimal adversary is balanced without loss of generality, it is immediate that a uniformly random adversary is optimal in this setting.~\cite{AbernethyWY08} consider a different variant of experts problem where the game stops when cumulative loss of any expert exceeds given threshold. Here too there is a clear candidate for the optimal adversary: the same as in~\citeauthor{LS14}, namely, pick an expert uniformly at random at every step. They specify optimal algorithm in terms of the underlying random walk.
The notable distinction of both~\cite{LS14,AbernethyWY08} from our setting is that their adversary is simple and static, i.e., it does not depend on the prior history. The random process to be analyzed in their setting is a standard random walk in $\Z^k$, while the random process in our setting even for $k=3$ is non-trivial.~\cite{ABRT08} consider general convex games and compute the minimax regret exactly when the input space is a ball, and show that the algorithms of~\cite{Zinkevich03} and~\cite{HKKA06} are optimal w.r.t. minimax regret.~\cite{AABR09} provide upper and lower bounds on the regret of an optimal strategy for several online learning problems without providing algorithms, by relating the optimal regret to the behavior of a certain stochastic process.~\cite{MS10} consider a continuous experts setting where the algorithm knows beforehand the maximum number of mistakes of the best expert.~\cite{RST10} introduce the notion of sequential Rademacher complexity and use it to analyze the learnability of several problems in online learning w.r.t. minimax regret.~\cite{RST11} use the sequential Rademacher complexity introduced in~\cite{RST10} to analyze learnability w.r.t. general notions of regret (and not just minimax regret).~\citet{RSS12} use the notion of conditional sequential Rademacher complexity to find relaxations of problems like prediction with static experts that immediately lead to algorithms and associated regret guarantees. They show that the random playout strategy has a sound basis and propose a general method to design algorithms as a random playout. In our work, we show that random playout (probability matching) is not just a good strategy, but it is optimal, for the case of $k=2,3$ experts.~\citet{Koolen13} studies the
regret w.r.t. every expert, rather than just the best expert in hindsight and considers tradeoffs in the pareto-frontier.~\cite{MA13} characterize the minimax optimal regret for online linear optimization games as the supremum over the expected value of a function of a martingale difference sequence, and similar characterizations for the minimax optimal algorithm and the adversary.~\cite{MO14} study online linear optimization in Hilbert spaces and characterize minimax optimal algorithms. 

\section{Preliminaries}
\label{sec:prelim}
\paragraph{Adversary.} The adversary at each time $t$ increases the gain of
expert $i \in \{1,2,\dots,k\}$ by a value $\gainit\in[0,1]$.  Thus adversary decides on
$\{\gainit[_{i\in[k]}][t]\}_{t=1}^{t=T}$. In particular, for each time $t$ the
adversary decides on the distribution $\dist_t$ to draw $\gainst$ from.  In
general, the adversary could be adaptive: i.e., $\dist_t$
could depend, apart from the history of gains $\gainsht$ till time $t-1$, also on the player's
past choices. But for the experts problem, it is known (Lemma 4.1 in~\cite{BL06}) that an {\em
oblivious adversary}, whose distribution $\dist_t$ at time $t$ is a function
only of $\gainsht$, is equally powerful\footnote{For the case of adaptive
adversary, there is an alternative definition of regret known as policy
regret~\cite{ADT12}, where this reduction does not apply. However in our setting, we don't use policy regret as it is too powerful and results in a linear regret. Also, for the bandits setting, where the player gets the feedback only about the gains of his chosen action (and not of every action) it is unknown 
whether adaptive adversaries are any more powerful than oblivious adversaries.}. Thus we focus
on oblivious adversaries from now on. We denote the joint distribution for
all $t \leq T$ as $\dist$. We denote the cumulative gain till time $t$ of
expert $i$ by $\Gainit[i]=\sum_{s=1}^{t}\gainit[i][s]$.  We denote the vector
of cumulative gains at time $t$ by $\Gainst=(\Gainit[1],\dots,\Gainit[k])$, and
denote the entire history of cumulative gains by $\Gainsh$.

\paragraph{Player.} Before making his decision at time $t$, the player
observes all prior history, that is $\gainsht$, but doesn't observe $\gainit$.
He decides on which expert to
follow, and, if the player follows expert $i$, he gains $\gainit$
at the end of step $t$. Specifically, the player
decides on the distribution $\algt$ over experts $\{1,\dots,k\}$. In general, the
player could be adaptive: i.e., his distribution $\algt$ could depend, apart
from $\gainsht$, on his own past choices. But an {\em oblivious player}, whose distribution
$\algt$ at time $t$ is a function only of $\gainsht$, is equally powerful. Thus
we focus on oblivious players from now on. We use $\algt(\gainsht)$ to denote
the gain of player at time $t$.

\paragraph{Regret.} The stopping time $T$ is known to both the algorithm and
the adversary. If the adversary chooses $\gainsht[T]$ and the player plays
$\alg$, the regret is given by the expression:

\be
\regretf(\gainsht[T],\alg) = \max_{i\in[k]}\Gainit[i][T]-\sum_{t=1}^{T}
\Ex{\algt\InParentheses{\gainsht}}.
\label{eq:regret_finite}
\ee
If the adversary uses a distribution $\dist$, the regret is given by
$\regretf(\dist,\alg)=\Exlong[\gains_{[0,T]}\sim\dist]{\regretf(\gainsht[T],\alg)}$.

\paragraph{Minimax regret.} The worst-case regret a player playing $\alg$
could experience is $\sup_{\dist}\regretf(\dist,\alg)$. Hence a robust
guarantee on the player's regret would be to optimize over $\alg$ for
worst-case regret, namely, $\inf_{\alg}\sup_{\dist}\regretf(\dist,\alg)$. This
is also referred to as the player's minimax regret as he tries to minimize his
maximum regret.

\paragraph{Binary adversary.} It turns out that an adversary that sets gains in
$\{0,1\}$ (that we call as a binary adversary) is as powerful as an adversary
that sets gains in $[0,1]$ (much like Theorem 10 in Luo and Schapire~\cite{LS14}). Formally, let $\dist^{[0,1]}$ be an arbitrary
adversary distribution with gains in $[0,1]$ and let $\dist^{\{0,1\}}$ be an
arbitrary adversary distribution with gains in $\{0,1\}$. Basically, we show that
$\inf_{\alg}\sup_{\dist^{\{0,1\}}}\regretf(\dist^{\{0,1\}},\alg) =
\inf_{\alg}\sup_{\dist^{[0,1]}}\regretf(\dist^{[0,1]},\alg)$ (see Claim~\ref{cl:binary} in Appendix~\ref{app:prelim}). From now on,
without loss of generality, we focus only on binary adversaries.

\paragraph{Minimax theorem.}
Our setting is naturally seen as a two player zero-sum game between the
player and the adversary. The player and the adversary, though online in nature, can be
described entirely upfront, i.e., by describing their (randomized) actions for
every possible history. The set of deterministic strategies for the player is a (huge) finite set, and hence the set of player's randomized strategies is a (huge) simplex.
Similarly, the set of adversary's randomized strategies is a (huge) simplex. The regret
function is a bilinear function in the player's and adversary's strategies. Thus, the
$\inf$ and $\sup$ can be replaced by $\min$ and $\max$, and the famous minimax theorem due to von Neumann~\cite{vN28} applies, telling us that the minimax
regret of the game is given by
\be
\min_{\alg}\max_{\dist} \InBrackets{\regretf(\dist,\alg)}=
\max_{\dist}\min_{\alg} \InBrackets{\regretf(\dist,\alg)}.
\label{eq:minimax}
\ee

We refer to the optimal algorithm that defines the LHS as the {\em minimax 
optimal algorithm} and similarly, the optimal adversary that defines the RHS as 
the {\em minimax optimal} adversary. The minimax optimal algorithm $\alg^*$ and 
the minimax optimal adversary $\dist^*$ form a Nash equilibrium: that is, they 
are
mutual best responses.

\paragraph{Balanced adversary.} We show that the minimax optimal adversary can, 
without loss of generality, be ``balanced'' (see Claim~\ref{cl:arm_balanced} in 
Appendix~\ref{app:prelim}). In other words,
for every time $t$, irrespective of what the history $\gainsht$ is, the minimax optimal adversary can pick $\dist_t$ such that $\Exlong[\dist_{t}(\gainsht)]{\gainit}$ is the same for each expert $i$. I.e., the expected gains of all experts are equal at every step, irrespective of history.

\paragraph{Dependence on cumulative gains.}

The minimax optimal algorithm can also choose the distribution $\algt$ at time $t$, based only on $\Gainst[t-1]$ instead of $\gainsht$. Henceforth, we focus on such algorithms and adversaries, and denote the time $t$ distributions by $\algt(\Gainst[t-1])$ and $\dist_t(\Gainst[t-1])$ respectively.

\begin{claim}
For any balanced adversary $\dist$ all algorithms will result in the same regret for the player. In particular, focusing on the algorithm that chooses an expert uniformly at random at every time $t$, the regret inflicted by $\dist$ is given by
\begin{align*}
\regretf(\dist,\alg) &= \regretf(\dist)\\
& = \Exlong[\gains_{[0,T]}\sim\dist]{
\max_{i\in[k]}\Gainit[i][T]-\frac{\sum_{i\in[k]}\Gainit[i][T]}{k}
}.
\end{align*}
Given that a minimax optimal adversary $\dist$ can always be balanced, the minimax optimal regret is given by $\regretf(\dist)=\Exlong[\gains_{[0,T]}\sim\dist]{\max_{i\in[k]}\Gainit[i][T]-\frac{\sum_{i\in[k]}\Gainit[i][T]}{k}}$.
\label{cl:alg_balanced}
\end{claim}

\paragraph{Geometric horizon.} We introduce the (almost identical) notation that we use for the geometric horizon setting in Section~\ref{sec:delta}.

\section{Finite horizon}
\label{sec:finite}

\paragraph{Two experts: optimal adversary and regret.}
The optimal regret in the finite horizon setting for the case of $k=2$ was derived by~\citet{Cover65}, showing that as $T\to\infty$, the optimal regret approaches $\sqrt{\frac{T}{2\pi}}$. While~\citeauthor{Cover65} also gave an expression (involving a sum and binomial coefficients) for the algorithm's probabilities, getting just the optimal adversary and the optimal regret value of $\sqrt{\frac{T}{2\pi}}$ is simpler. We begin by rewriting the expression for the regret from Claim~\ref{cl:alg_balanced} for the case of two experts.
\begin{align}
\regretf(\dist) &= \Exlong[\gains_{[0,T]}\sim\dist]{
\max_{i=1,2}\Gainit[i][T]-\frac{\Gainit[1][T]+\Gainit[2][T]}{2}
}\nonumber\\
&= \Exlong[\gains_{[0,T]}\sim\dist]{
\frac{|\Gainit[1][T]-\Gainit[2][T]|}{2}
}\nonumber\\
&= \frac{1}{2}\Exlong[\gains_{[0,T]}\sim\dist]{\Big\lvert\sum_{t=1}^{T}(\gainit[1][t]-\gainit[2][t])\Big\rvert}
\label{eq:2_experts}
\end{align}
The adversary's optimization problem now is to construct these $\gainit$'s such
that they maximize the RHS of equation~\eqref{eq:2_experts}, subject to being
balanced. This problem is equivalent to the problem of designing a
one-dimensional random walk, that respects the constraint that the probability
of jumping one step left and one step right are the same, and maximizes the
absolute distance from the origin. This equivalence is obtained by interpreting
$\gainit[1][t] - \gainit[2][t]$ as the random-walk variable (which can take
values of $-1,1$ and $0$), and the condition of being balanced translates to
the constraint of jumping left and right with equal probability. We emphasize
that the adversary has a control over this random walk, i.e., he can decide
separately on the probabilities of jumping left or right at every time step $t$
and every vector of gains $\Gainst$ with the restriction to be unbiased
towards jumping left or right. Being balanced means that the only design choice
left is the probability of staying still (not jumping left or right). To
maximize the absolute distance from the origin, the latter probability has to
be zero.  Indeed, the adversary may as well postpone all his ``staying still''
turns until the deadline $T$. In such a case remaining still for the last few
steps is not better in expectation than doing random walk. {\em Thus, the optimal
adversarial strategy in the 2 experts case is: at every step, choose an expert
uniformly at random (with probability 1/2) and set him to $1$, and the other
expert to $0$.}

Given the optimal adversarial strategy description above, the optimal regret in 
the finite horizon model with $T$ steps is exactly half the expected distance 
travelled by a simple random walk in $T$ steps, which approaches $\sqrt{\frac{T}{2\pi}}$ 
as $T \to \infty$. Thus, $\regretf(\dist) \to \sqrt{\frac{T}{2\pi}},\text{ when } T \to\infty$.

\paragraph{Two experts: optimal probability matching algorithm.} 
It turns out that the optimal algorithm is precisely a probability matching 
algorithm, i.e, the algorithm picks each expert with the probability that the 
respective expert finishes in the lead (we break possible ties in favor of the 
unique expert who does not have any expert ahead of him in each of the last two 
steps).

We derive this from an explicit correspondence between simple random walk and the minimax 
regret value of games with any given initial configuration of expert cumulative gains.  
The formal argument is given in Appendix~\ref{app:pmatching}. The probability matching 
interpretation allows us to give the following simple and explicit description of the 
optimal algorithm for two experts in the finite horizon model.

\begin{theorem}
\label{thm:2pmatching}
Let $k$ be the number of remaining time steps and let $X$ be a random variable
with a binomial distribution $\binomdist$, when $k$ is odd and $\binomdist[k-1]$, when $k$ is even. 
The optimal algorithm computes the difference $d (\ge 0)$ of cumulative gains
between the leading and lagging expert and chooses them with probabilities
$p_1(d) = \Prx{X-\Ex{X}<d}$, and $p_2(d) = \Prx{X-\Ex{X}>d}$.
\end{theorem}

\paragraph{$k$ experts: optimal adversary and regret.}
The simplification afforded by the 2 experts case doesn't carry through for
arbitrary $k$. In Appendix~\ref{app:finite} we have a detailed technical description of the adversary's problem and useful observations about them, including the proofs of claims~\ref{cl:polytope} and~\ref{cl:finiteVertices} below.

\begin{claim}
For each time step $t$, the set of all possible distributions
$\dist_{t}(\Gainst[t-1])$ for a balanced adversary forms a convex polytope in
$2^k$-dimensional space.
\label{cl:polytope}
\end{claim}
\begin{claim}
There is a fixed finite set of distributions (over $2^k$ actions) such that at
every time step $t$ and every previous history, the minimax optimal adversary can always choose a distribution
from this set.
\label{cl:finiteVertices}
\end{claim}
\paragraph{$k$ experts: optimal algorithm.}
We refer the reader to appendix~\ref{app:finite} for an expanded version of the 
discussion in this subsection. The main algorithmic question in the finite 
horizon case is if there is a simple description of the optimal algorithm. For 
the case of $k=2$ experts, we show that the answer is yes: the optimal 
algorithm is a simple probability matching algorithm. I.e., the optimal 
algorithm follows expert $i$ with the probability that, given the current 
cumulative gains of both the experts and the number of remaining steps, expert 
$i$ will finish as the leading expert. The derivation of this optimal algorithm 
is related to how we derive the optimal algorithm for the geometric horizon 
case. So we do this in Appendix~\ref{app:2experts} along with the derivation 
for geometric horizon.

\section{Geometric horizon}
\label{sec:delta}
\paragraph{Minimax theorem for the geometric horizon model.} We use the same 
notation
for the geometric horizon model and the finite horizon model except that we use $\regretd$ for
regret in the geometric model instead of $\regretf$. Our setting in the
geometric model is again a two player zero-sum game between the player and
the adversary, although the game is not finite now. But a slight generalization
of von Neumann's minimax theorem guarantees that the minimax relation we need is still true. For any bilinear function $f(x,y)$ defined on $\mathcal{X}\times\mathcal{Y}$, where $\mathcal{X}$ and $\mathcal{Y}$ are convex and compact sets\footnote{the space of pure strategies for the adversary consists of all infinite sequences of vector gains $\{v_1,v_2,\dots\}$, where each $v_t \in [0,1]^k$. This space is compact, for example, in a normed space $\twonorm{\{v_t\}_{t=1}^{\infty}} = \sum_t \twonorm{v_t}^2 / t^2$. Note that in such a normed space the regret is still a continuous function of the sequence $\{v_t\}_{t=1}^{\infty}$: this is because the geometric-infinite horizon results in a discount of $(1-\delta)^t$ for round t’s utilities, and this decays much faster than $1/t^2$.}, we have $\inf_{x\in\mathcal{X}}\sup_{y\in\mathcal{Y}}f(x,y) = \sup_{y\in\mathcal{Y}}\inf_{x\in\mathcal{X}}f(x,y)$.  In our case, the space of strategies of the adversary and the algorithm can be easily shown to be convex compact sets. Thus, we have the expected minimax regret of the game given by: 

\be
\inf_{\alg}\sup_{\dist} \InBrackets{\regretd(\dist,\alg)}=
\sup_{\dist}\inf_{\alg} \InBrackets{\regretd(\dist,\alg)}.
\label{eq:delta_minimax}
\ee

\paragraph{Preliminary claims on the geometric horizon model.}
 All the claims for the finite horizon model, namely, 
Claims~\ref{cl:binary},~\ref{cl:arm_balanced},~\ref{cl:alg_balanced},~\ref{cl:polytope}
and~\ref{cl:finiteVertices}, also carry over to the geometric horizon model with
appropriate modifications. We state the modified claims as 
Claim~\ref{cl:deltaAll} in Appendix~\ref{app:delta}.

We now derive the optimal adversary, regret and algorithm for the case of $2$ experts in appendix~\ref{app:2experts}. We state our results here.

\begin{theorem}
\label{thm:2opt}
In the geometric horizon model for $2$ experts with parameter $\delta \in (0,1)$:
\begin{enumerate}
\item The optimal adversary, at every time step, advances the leading expert alone with probability
$\frac{1}{2}$ and lagging expert alone with probability $\frac{1}{2}$.
\item The optimal regret is $\frac{1-\delta}{2\sqrt{1-(1-\delta)^2}} \to
\frac{1}{2}\frac{1}{\sqrt{2\delta}}$ as $\delta \to 0$. 
\item The optimal algorithm, at every time step, computes the difference $d
(\geq 0)$ of cumulative gains
between the leading and lagging expert, and chooses them with probabilities
$p_1(d) = 1-\frac{1}{2}\xi^d$, and $p_2(d) = \frac{1}{2}\xi^d$. Here $\xi = \frac{1-\sqrt{1-(1-\delta)^2}}{1-\delta} \sim 1-\sqrt{2\delta}$. 
\end{enumerate}
\end{theorem}

We derive the optimal adversary, regret and algorithm for the case of $3$ experts in appendix~\ref{app:3experts}. This derivation is significantly more involved for $k=3$ when compared to $k=2$. We
state our results here.
\begin{theorem}
\label{thm:3opt}
In the geometric horizon model for $3$ experts with parameter $\delta \in (0,1)$:
\begin{enumerate}
\item The optimal regret is $\frac{2}{3}\frac{1-\delta}{\sqrt{1-(1-\delta)^2}} \to
\frac{2}{3}\frac{1}{\sqrt{2\delta}}$ as $\delta \to 0$.
\item The optimal algorithm, at every time step, computes the differences $d_{ij}$
between the cumulative gains of experts ($i$ denotes the expert with $i$th largest cumulative gains, and hence $d_{ij}\geq 0$ for all $i < j$). As a function of the $d_{ij}$'s the algorithm
follows the leading expert with probability $p_1(\mathbf{d}) = 1-\frac{\xi^{d_{12}}}{2} - \frac{\xi^{d_{13}+d_{23}}}{6}$,
the second expert with probability
$p_2(\mathbf{d}) = \frac{\xi^{d_{12}}}{2} - \frac{\xi^{d_{13}+d_{23}}}{6}$, and the
lagging expert with probability $p_3(\mathbf{d}) = \frac{\xi^{d_{13}+d_{23}}}{3}$. Here $\xi = \frac{1-\sqrt{1-(1-\delta)^2}}{1-\delta} \sim \sqrt{2\delta}$.
\item The optimal adversary, at every time step, computes the differences $d_{ij}$'s, and follows the following strategy as a function of the $d_{ij}$'s. Here strategy $\{1\}\{2\}\{3\}$ means exclusively advancing
expert $1$ (leading expert) with probability $1/3$, expert $2$ (middle expert) with probability $\frac{1}{3}$ and expert
$3$ (lagging expert) with probability $\frac{1}{3}$. Strategy $\{1\}\{23\}$ means advancing expert $1$
alone with probability $\frac{1}{2}$ and experts $2$ and $3$ together with
probability $\frac{1}{2}$.
\begin{description}
\item[$0 < d_{12} < d_{13}:$]  $\{1\}\{23\}$ (any mixture of $\{1\}\{23\}$ with $\{13\}\{2\}$ would also work).
\item[$ 0 = d_{12} < d_{13}:$]  $\{1\}\{23\}$ (any mixture of $\{1\}\{23\}$ with $\{13\}\{2\}$ would also work).
\item[$0 < d_{12} = d_{13}:$] $\{1\}\{23\}$ (any mixture of $\{1\}\{23\}$ with $\{1\}\{2\}\{3\}$ would also work).
\item[$0 = d_{12} = d_{13}:$] $\{1\}\{2\}\{3\}$.
\end{description}
\end{enumerate}
\end{theorem}

\paragraph{Interpretation as a probability matching algorithm.} We show that 
the optimal algorithms for $k=2,3$ can be interpreted as following each expert 
with the probability he finishes as the leader (probability matching) when 
following an optimal adversary. We prove this respectively in 
Appendix~\ref{app:2experts} and~\ref{app:3experts}.

\section{Connections to random walk}
\label{sec:brownian}
We already saw for the case of two experts that the optimal strategy for the
adversary has a direct connection to random walk. In this section we study
larger number of experts, and show that this connection is deep and extends to
nontrivial aspects of random walk. We state our results here and prove them (Theorems~\ref{thm:3brownian} and~\ref{thm:4brownian}) in the full version~\cite{GPS14}. 

\paragraph{Regret Lower Bounds for $k=3, 4$ experts.}
While we already have shown in Section~\ref{sec:delta} that the optimal
regret in the case of $3$ experts is $\frac{2}{3}\frac{1}{\sqrt{2\delta}}$ as
$\delta \to 0$, the adversary we used there was not the comb adversary. Here we derive the same regret
through the comb adversary. Next, we analyze the comb adversary for $k=4$ experts and show that as $\delta\to 0$ it inflicts a regret that is asymptotic to $\frac{\pi}{4}\frac{1}{\sqrt{2\delta}}$. 

\begin{theorem}
\label{thm:3brownian}
The regret inflicted by the adversary that advances experts 1 and 3 together with probability
$\frac{1}{2}$, and, expert $2$ with probability $\frac{1}{2}$, as $\delta \to 0$, is $\frac{2}{3}\frac{1}{\sqrt{2\delta}}$. 
\end{theorem}

\begin{theorem}
\label{thm:4brownian}
The regret inflicted by the adversary that advances experts 1 and 3 together with probability
$\frac{1}{2}$, and, experts $2$ and $4$ together with probability $\frac{1}{2}$, as $\delta \to 0$, is $\frac{\pi}{4}\frac{1}{\sqrt{2\delta}}$. 
\end{theorem}

\paragraph{Main idea behind the analysis.} We show a bijection between the random process defined by the comb adversary and the simple random walk of a particle between two walls. For $k=3$, the two walls are ``movable'', while for $k=4$, one wall is `'fixed'' and the other is movable. I.e., when the particle coincides with the wall and tries to penetrate it in the next step, a movable wall moves one step in the direction of particle's movement while the particle doesn't move, but a fixed wall doesn't move and the particle bounces one step back. The comb adversary's regret maps to half of the expected number of visits of the particle to one of the movable walls for $k=3$, and the fixed wall for $k=4$.  Computing the expected number of visits leads to interesting asymptotic analysis.

\bibliographystyle{plainnat}
\bibliography{bibs,machine_learning}

\appendix
\section{Proofs from Section~\ref{sec:prelim}}
\label{app:prelim}

\begin{claim}[Binary adversary]
The minimax regret defined by the class of binary adversaries is exactly the
same as that defined by general adversaries:
\[\inf_{\alg}\sup_{\dist^{\{0,1\}}}\regretf(\dist^{\{0,1\}},\alg) =
\inf_{\alg}\sup_{\dist^{[0,1]}}\regretf(\dist^{[0,1]},\alg).\]
\label{cl:binary}
\end{claim}
\begin{proof}
Given that the class of general adversaries is larger than the class of binary
adversaries, it immediately follows that
$\inf_{\alg}\sup_{\dist^{\{0,1\}}}\regretf(\dist^{\{0,1\}},\alg) \leq
\inf_{\alg}\sup_{\dist^{[0,1]}}\regretf(\dist^{[0,1]},\alg)$. It is therefore
enough to show that
$\inf_{\alg}\sup_{\dist^{\{0,1\}}}\regretf(\dist^{\{0,1\}},\alg) \geq
\inf_{\alg}\sup_{\dist^{[0,1]}}\regretf(\dist^{[0,1]},\alg)$. This can be seen
as follows: consider the minimax optimal algorithm $\alg^*$ for the class of
binary adversaries.  When faced with a $[0,1]$ adversary, $\alg^*$, in every
round, ``discretizes'' the gains set by the adversary by independently rounding
them to $0$ or $1$ so that the expectation is equal to the gain $\gainit$ set by
the adversary: i.e., a gain of $\gainit$ is set to $1$ with probability
$\gainit$ and $0$ with the remaining probability. From the algorithm $\alg^*$'s
point of view, whether the adversary originally used a distribution with gains
in $[0,1]$ that $\alg^*$ discretized to $\{0,1\}$, or the adversary already
set gains in $\{0,1\}$ with the {\em same} distribution doesn't make a
difference. Both result in exactly the same expected gains for the algorithm.
However, using the discretized version could possibly help the adversary. We see
this as follows.

For some step $t$ and history $\gainsht$, let the adversary set
expert's $i$ gain to be $\gainit \notin \{0,1\}$ with non zero
probability. Consider the following step-by-step discretization by the adversary.
It changes random variable $\gainit$ (only for expert $i$ and time $t$ and
history $\gainsht$) to be $\{0,1\}$ while preserving expectations. While performing
this discretization the adversary does not change the distribution in
future steps, i.e., it chooses future distributions as if the discretization was
not performed. We now show that the expected gain of the best expert can only increase. For
each fixed value $\gainit$ let us denote by $\xi$ a random
variable that takes value 1 with probability $\gainit$ and 0 with probability
$1-\gainit$. Let us fix all choices of the adversary other than $\xi$.
Then our substitution of constant $\gainit$ by a random variable $\xi$ can only
increase the gain of the best expert $\max_{i\in[k]}\Gainit[i][T]$. Indeed,
this follows from the inequality
\[
\max(\Ex{\xi}+c_1,c_2) \le \Ex{\max(\xi+c_1,c_2)},
\]
where $c_1$ and $c_2$ are two constants determined by a fixed set of adversary's
random choices. Hence, our modification may only increase the total expected
regret, proving the theorem.
\end{proof}

\begin{claim}[Balanced adversary]
For each time $t$ and for every possible history $\gainsht$, the minimax optimal adversary can pick $\dist_{t}(\gainsht)$, such that
$\Exlong[\dist_{t}(\gainsht)]{\gainit}$ is the same for each expert $i$ .
\label{cl:arm_balanced}
\end{claim}
\begin{proof}
Given an adversary that is not balanced, we modify it so that
algorithm cannot improve, but the expected gain of the best expert
$\max_{i\in[k]}\Gainit[i][T]$ may only increase.  For the minimax optimal adversary $\dist$,
one best response algorithm is to choose an expert

\[
i^*\in\argmax_{i\in[k]}\Exlong[\dist_{t}(\gainsht)]{\gainit}.
\]
The adversary can modify distribution $\dist_{t}(\gainsht)$ so that for all experts $[k]\setminus\{i^*\}$

\[
\Exlong[\dist_{t}(\gainsht)]{\gainit}= \Exlong[\dist_{t}(\gainsht)]{\gainit[i^*][t]},
\]
by switching some of the gains from $0$ to $1$ for
$i\in[k]\setminus \{i^*\}$. While making such transformation the adversary does
not change $\dist$ in the future time steps after $t$, i.e., the adversary continues as
if there was no transformation at time $t$. The adversary also reveals to the
algorithm the value of $\gainit$ as it was drawn in the original
$\dist_{t}(\gainsht)$.

The best response algorithm described above cannot improve its gain at time $t$, as the expected gain of
the best expert does not change. We also note that the algorithm cannot improve
in time before $t$, nor it can improve for the time steps after $t$,
as the knowledge of the algorithm about prior history and the adversary
distribution do not change for these times.

On the other hand, the expected gain of the best expert
$\max_{i\in[k]}\Gainit[i][T]$ could only improve for every such modification of
$\dist_{t}(\gainsht)$.
\end{proof}

\section{Proofs and Results from Section~\ref{sec:finite}}
\label{app:finite}
\subsection{k experts, finite horizon: optimal adversary and regret}
As mentioned before, the simplification afforded by the 2 experts case doesn't carry through for
arbitrary $k$. Here is the design problem faced by the optimal adversary: for every
time step $t$, given the gains $\Gainst[t-1]$ at time $t-1$, the adversary has
to compute the distribution $\dist_{t}(\Gainst[t-1])$ at time $t$ so as to
maximize the expression for regret given by
$$\regretf(\dist) = \Exlong[\gains_{[0,T]}\sim\dist]{
\max_{i\in[k]}\Gainit[i][T]-\frac{\sum_{i\in[k]}\Gainit[i][T]}{k}
}.$$
Note that given any vector of gains $\Gainst[t-1]$ after $t-1$ time steps, the
adversary's distribution $\dist_t$ at time $t$ is over $2^k$ actions
corresponding to ``setting gain to 0'' or ``setting gain to 1'' for each expert
with the restriction that the expected gain of each expert is the same. This
design problem of the adversary can be thought of as the design of a controlled
random walk on $\Z^{k}$ so that the advance in each dimension in expectation is
the same at every step, with the objective of maximizing the regret expression
above.

\begin{oneshot}{Claim~\ref{cl:polytope}}
For each time step $t$, the set of all possible distributions
$\dist_{t}(\Gainst[t-1])$ for a balanced adversary forms a convex polytope in
$2^k$-dimensional space.
\end{oneshot}
\begin{proof}
First, note that if two distributions are feasible, a convex combination of
them is also feasible. Thus the set of feasible distributions is convex.
Second, the feasibility conditions can all be described with linear
equalities/inequalities. Finally, the set of feasible distributions is bounded.
Thus the set of feasible distributions is a convex polytope.
\end{proof}

\begin{oneshot}{Claim~\ref{cl:finiteVertices}}
There is a fixed finite set of distributions (over $2^k$ actions) such that at
every time step $t$ and every previous history, the minimax optimal adversary can always choose a distribution
from this set.
\end{oneshot}
\begin{proof}
Given that the set of possible distributions is a convex polytope, at every
$t$, it is a weakly dominant strategy for the adversary to choose from one
among the vertices of this polytope. This is because the expression for regret
(which the adversary maximizes) is linear in distributions, i.e., a convex
combination of two distributions will yield a regret which is the convex
combination of the corresponding regrets. Furthermore, this convex polytope of
possible distributions remains the same, independent of $t$ and previous
history.
\end{proof}

\begin{remark}
Note that this polytope of possible distributions has exponentially many vertices. This is easy to see: for every subset $S$ of $\{1,\dots,k\}$, treat experts in $S$ as a group and those in $\bar{S}$ as a group. With probability half, set the gains of experts in $S$ to be $1$ and those in $\bar{S}$ to be $0$, and with the remaining probability do the opposite. Each such distribution is a vertex, and there are exponentially many of them.
\end{remark}

\begin{remark}
For concreteness, for the case of $k=3$ and $k=4$, we list all the vertices
of the distribution polytopes. While describing a distribution, we shall list
only actions in its support, as it turns out the respective probabilities can
be reconstructed from the balanced condition for any extremal distribution in
our convex polytope. While describing an action, we list the set of experts
whom we advance. For instance, the list $\{1\},\{23\}$ reads as ``advance
expert $1$ with probability $0.5$; advance experts $2$ and $3$ (but not $1$)
with remaining probability''. Similarly, the list $\{234\}\{12\}\{13\}\{14\}$
reads as ``with probability $2/5$ advance experts 2,3, and 4; with probability
$1/5$ advance experts 1 and 2; with probability $1/5$ advance experts 1 and 3;
with probability $1/5$ advance experts 1 and 4.''
For $k=3$ and $k=4$ the lists are (excluding the trivial distribution $\{\}$ that advances no experts at all, the distributions $\{123\}$ for $k=3$ and $\{1234\}$ for $k=4$ that advance all the experts together):
\begin{center}
\begin{tabular}{|c|c|c|}
\hline
\multicolumn{2}{ |c| }{$k=4$} \\
\hline
$\{123\}\{4\}$                 &  $\{1\}\{2\}\{34\}$   \\
$\{124\}\{3\}$                 &  $\{1\}\{3\}\{24\}$   \\
$\{134\}\{2\}$                 &  $\{1\}\{4\}\{23\}$   \\
$\{234\}\{1\}$                 &  $\{2\}\{3\}\{14\}$   \\
                                     &  $\{2\}\{4\}\{13\}$   \\
$\{12\}\{34\}$                 &  $\{3\}\{4\}\{12\}$    \\
$\{13\}\{24\}$                 &  \\
$\{14\}\{23\}$    		         &  $\{12\}\{134\}\{234\}$  \\
					         &  $\{13\}\{124\}\{234\}$\\
$\{1\}\{23\}\{24\}\{34\}$     &   $\{14\}\{123\}\{234\}$\\
$\{2\}\{13\}\{14\}\{34\}$     &   $\{23\}\{124\}\{134\}$\\
$\{3\}\{12\}\{14\}\{24\}$     &   $\{24\}\{123\}\{134\}$\\
$\{4\}\{12\}\{13\}\{23\}$     &   $\{34\}\{123\}\{124\}$\\
						  &  \\
$\{123\}\{124\}\{134\}\{234\}$ &  $\{123\}\{14\}\{24\}\{34\}$  \\
                                            &  $\{124\}\{13\}\{23\}\{34\}$\\
$\{1\}\{2\}\{3\}\{4\}$             &  $\{134\}\{12\}\{23\}\{24\}$\\
                                            &  $\{234\}\{12\}\{13\}\{14\}$ \\                                       
\hline
\end{tabular}
\end{center}

\begin{center}
\begin{tabular}{|c|}
\hline
$k=3$                    \\ \hline
$\{1\}\{23\}$            \\
$\{2\}\{13\}$            \\
$\{3\}\{12\}$            \\
$\{1\}\{2\}\{3\}$        \\
$\{12\}\{13\}\{23\}$     \\
\hline
\end{tabular}
\end{center}
\end{remark}

Encouraged by a very simple optimal adversary for $k=2$, one
may think that similar behavior extends to $3$ or more experts.
Unfortunately, this is not the case. The optimal adversary will be time
dependent for $k=3$. For instance, if only one step remains before deadline
the optimal adversary would do the following:

\begin{itemize}
\item if $\Gainit[1][T-1]=\Gainit[2][T-1]=\Gainit[3][T-1]$, then $\{1\}\{2\}\{3\}$;
\item if $\Gainit[1][T-1]=\Gainit[2][T-1]>\Gainit[3][T-1]$, then $\{1\}\{23\}$ or $\{13\}\{2\}$;
\item if $\Gainit[1][T-1]>\Gainit[2][T-1]$, then any balanced strategy.
\end{itemize}

\subsection{k experts, finite horizon: optimal algorithm}
We note that given a finite time horizon $T$ and finite list of balanced
distributions for the adversary, one can write a dynamic program for the
maximal value of the regret at any time period $t\le T$ and initial vector of
gains $\Gainst\in[T]^k$. We can solve this program by using backward induction
over time and furthermore given the regret function at every time step
$t\in[T]$ and vector of gains $\Gainst\in[T]^k$ we can compute the best
strategy for the algorithm. The running time of such an algorithm would be
$O(T^k)$.  This approach gives us the answer for a small number of experts
and reasonably small time horizon $T$. On the other hand, it becomes impractical as $T$ and
especially $k$ get larger, and furthermore, it does not tell us much about intrinsic
structure of the optimal algorithm and the optimal adversary.

Even for $k=2$ the optimal algorithm depends
on the time remaining before the deadline $T$.  For example, if the leading
expert is ahead of the lagging expert by more than the number of remaining time
steps, then the optimal algorithm should always choose the leading expert; on
the other hand, if the difference between leading and lagging experts is
smaller than the time remaining, then there should be non zero chance of
selecting the lagging expert.

\paragraph{The probability matching algorithm.} Given this, the main question in the finite horizon case if there is a simple description of the optimal algorithm. For the case of $k=2$ experts, we show that the answer is yes: the optimal algorithm is a simple probability matching algorithm. I.e., the optimal algorithm follows expert $i$ with the probability that, given the current cumulative gains of both the experts and the number of remaining steps, expert $i$ will finish as the leading expert.

The derivation of this optimal algorithm is related to how we derive the optimal algorithm for the geometric horizon case. So we do this towards the end of Section~\ref{app:pmatching}.

\section{Geometric Horizon Model}
\label{app:delta}
\begin{claim}{Observations on the geometric horizon model.}
\label{cl:deltaAll}
The following statements are true:
\begin{enumerate}
\item The minimax regret defined by the class of binary adversaries is exactly the
same as that defined by general adversaries:
$\inf_{\alg}\sup_{\dist^{\{0,1\}}}\regretd(\dist^{\{0,1\}},\alg) =
\inf_{\alg}\sup_{\dist^{[0,1]}}\regretd(\dist^{[0,1]},\alg).$
\item For each time step $t$ and for every possible history $\gainsht$, the minimax optimal adversary can pick $\dist_{t}(\gainsht)$, such that
$\Exlong[\dist_{t}(\gainsht)]{\gainit}$ is the same for each expert $i$ .
\item A balanced adversary $\dist$ inflicts the same regret on every algorithm $\alg$. Since the minimax optimal adversary can always be balanced, the minimax optimal regret is given by:
\begin{align*}
&\regretd(\dist,\alg)\\
&=\regretd(\dist)\\
& = \sum_{t=0}^{\infty}\delta\cdot(1-\delta)^t\\
&\qquad\qquad\Exlong[\gains_{[0,t]}\sim\dist]
        {\left(\max_{i\in[k]}\Gainit-\frac{\sum_{i\in[k]}\Gainit}{k}\right)}
\end{align*}
\item For each time $t$, the set of all possible distributions $\dist_{t}(\Gainst[t-1])$
      for the adversary forms a convex polytope with exponentially many (in $k$) vertices .
\item There is a fixed finite set of distributions (over $2^k$ actions) such that at every time $t$ and every previous history, the minimax optimal adversary can always choose a distribution from this set.
\end{enumerate}
\end{claim}
\begin{remark}
The third point in Claim~\ref{cl:deltaAll} above says that the minimax optimal adversary makes all algorithms achieve the same regret. In particular, if the precise realization of the stopping time random variable was leaked to the algorithm, the minimax optimal regret is not influenced in any way. On the other hand, if the adversary knew the realization of the stopping time information, it could potentially increase the minimax optimal regret. This proves that the algorithm does not benefit from knowing the precise realization of the stopping time information, whereas the adversary could potentially benefit from it.
\end{remark}

\subsection{Two experts: optimal algorithm, adversary and regret}
\label{app:2experts}
\paragraph{Notational convention.} At each time step $t$ we always enumerate
experts in the decreasing order of their cumulative gains $\Gainst$, i.e.,
experts $1$ and $2$ don't refer to identities of experts but to the leading
expert and trailing expert respectively. Observe that the strategy of the
optimal adversary at any moment $t$ should not change if cumulative gains
$\Gainst$ of all experts are changed by the same amount for every expert.
Thereby, at every time step $t$ we shall always adjust the total gains
$\Gainst$ of our experts, so that the leading expert $1$ has zero cumulative
gain $\widehat{G}_{1t}=0$.  We denote the adjusted gain of the lagging expert
$\widehat{G}_{2t}=\Gainit[2][t]-\Gainit[1][t]$ by $x$ (note that $x\le 0$).

We denote by $f(x)$ the optimal regret the adversary can obtain for an initial
configuration of $\Gains=(0,x)$, i.e., leading expert has $0$ gain and lagging
expert has $x$ gain (again, recall that $x \leq 0$). The useful thing about
this notation is that if we start at $(0,x)$ for any $x$, and the game
immediately ends at that round, we get a regret of $0$: the max expert gain is
$0$, and the algorithm didn't get any chance to get any gain because the game
ended right away. So $0 - 0 = 0$ is the regret.

\paragraph{System of Equations.} We are now ready to write our system of equations connecting these $f(x)$'s.
Our discussion of~\citeauthor{Cover65}'s result in Section~\ref{sec:finite} showed that the minimax optimal adversary in the finite
horizon model was independent of the horizon $T$, and advanced expert $1$ or
$2$ mutually exclusively with probability $\frac{1}{2}$ each. The
independence from time horizon $T$ in the finite horizon model immediately
means that this adversary is also minimax optimal for the geometric horizon model: it
doesn't care when the game ends. This adversary advances the leading expert
with probability $\frac{1}{2}$ and lagging expert with probability
$\frac{1}{2}$. Thus, starting from the $(0,x)$ configuration, we go to the
$(1,x)$ configuration with probability $\frac{1}{2}$ (corresponds to adversary
advancing the leading expert), and go to the $(0,x+1)$ configuration with
probability $\frac{1}{2}$ (corresponds to adversary advancing the lagging expert).
In the meanwhile, the algorithm would have gained $1$ with probability
$\frac{1}{2}$ regardless of which expert was advanced.  This can be transcribed
to an equation right away except that the $(1,x)$ configuration is not in our
standard format: our standard format normalizes the largest gain to $0$. To
perform such a normalization here, notice that the paths of the optimal
adversary starting at $(1,x)$ and $(0,x-1)$ are indetical except that the ``max
expert gain'' is precisely one larger when starting from
$(1,x)$ than when starting from $(0,x-1)$. We take this into account in our
equations. Summarising this as an equation we get,
\begin{align*}
f(x) &= \delta\cdot 0 + (1-\delta)\\
&\qquad\cdot\left[\frac{1}{2}\left(f(x-1)+1\right) +
\frac{1}{2}f(x+1) - \frac{1}{2} \right] \\
&= (1-\delta)\cdot\left[\frac{f(x-1)+ f(x+1)}{2}\right].
\end{align*}

When $x$ is $0$ we have to take special care because $(0,x+1)$ is just $(0,1)$.
First we rewrite gains in the descending order to obtain the $(1,0)$
configuration. But this is not in starndard format: so we go to the $(0,-1)$
format and add a $1$ to the regret in this process. Thus the difference of
$(0,0)$ from $(0,x)$ is that normalization has to be done for both choices of
adversary, as against for just one choice. We get,
\begin{align*}
f(0) &= \delta\cdot 0 + (1-\delta)\\
&\qquad\cdot\left[\frac{1}{2}\left(f(0-1)+1\right) +
\frac{1}{2}\left(f(0+1)+1\right) - \frac{1}{2} \right]\\
& = (1-\delta)\cdot\left[f(-1)+\frac{1}{2}\right].
\end{align*}

Combining these two equations, we get the following system:
\begin{align}
f(x)&=(1-\delta)\cdot\frac{f(x-1)+f(x+1)}{2}\label{eq:f1}\\
f(0)&=(1-\delta)\cdot\left(f(-1)+\frac{1}{2}\right)
\label{eq:f2}
\end{align}

\paragraph{Optimal regret.} Thus we need to solve this recurrence relation for
$f(x)$. The characteristic polynomial of this recurrence is
$x^2-\frac{2}{1-\delta} x + 1=0$, which has two real roots $\xi_1 > 1 >\xi_2$,
and $\xi_1\cdot\xi_2 = 1$, given by
$\frac{1\pm\sqrt{1-(1-\delta)^2}}{1-\delta}$. The solution to our recurrence
relation is then of the form $f(x)=c_1\cdot\xi_1^x + c_2\cdot\xi_2^x$. As the
regret cannot grow faster than a linear function and cannot be negative, it
follows that $c_2$ must be 0. Combining $f(x) = c_1\xi_1^x$ with
equation~\eqref{eq:f2}, we get
$c_1\cdot\xi_1^0=(1-\delta)\cdot\left(c_1\cdot\xi_1^{-1}+\frac{1}{2}\right)$.
This gives us that $c_1=\frac{1}{\xi_1-\xi_2}$. The optimal regret is simply
the regret starting at $(0,0)$, which is given by $f(0)$. Thus the optimal
regret is $f(0) = c_1\xi_1^0 = c_1 = \frac{1}{\xi_1 - \xi_2} =
\frac{1-\delta}{2\sqrt{1-(1-\delta)^2}}$. Thus, as $\delta \to 0$, the optimal
regret $f(0) \to
\frac{1}{2}\frac{1}{\sqrt{2\delta}}$.

\paragraph{Optimal algorithm.} Note that because of minimax principle, we were able to
compute the precise regret without even knowing anything about the algorithm.
We now proceed to compute the optimal algorithm for
$k=2$.
This will reveal how even without knowing
the optimal adversary a priori, we can simultaneously discover both the optimal
adversary, optimal regret and the optimal algorithm (a useful
exercise to the significantly more complicated case of $k=3$).

Given configuration $(0,x)$ (with $x \leq 0$ as usual), the optimal algorithm
assigns probabilities $p_1(x)$ and $p_2(x) = 1-p_1(x)$ respectively for
choosing leading and lagging experts. We drop the arguments for probabilities
when it is clear from context. The adversary has four choices, namely
advancing expert $1$ alone, or expert $2$ alone, or both experts, or none of
the experts.  For $x<0$ this corresponds to decreasing $x$ by 1, increasing $x$
by 1, not changing $x$ for the last two choices. When $x =0$, we have to take
care of the fact that advancing $1$ alone and $2$ alone are similar, in that
both of them need the normalizing $+1$. Putting what we just described into
equations, we get (note the extreme RHS corner gives the adversary's actions
corresponding to each expression, and this is common for both $x <0$ and $x
=0$):

\begin{align}
f(x)=(1\text{-}\delta)\cdot\max
\begin{cases}
f(x-1) + 1 - p_1      & ~//\{1\}        \\
f(x+1) -p_2        & ~//\{2\}   \\
f(x) + 1-p_1 - p_2   &~//\{12\}      \\
f(x)     &~ //\{\}
\end{cases}
\label{eq:Regret2}
\end{align}
\begin{align}
f(0) = (1\text{-}\delta)\cdot\max
\begin{cases}
f(0-1)+1-p_1          &  ~ //\{1\}   \\
f(0-1)+1-p_2    &  ~ //\{2\}  \\
f(0)+1-p_1-p_2    &  ~ //\{12\}  \\
f(0)          &  ~ //\{\}
\end{cases}
\nonumber
\end{align}
We realize that $p_1(0) = p_2(0) = \frac{1}{2}$ by symmetry (the optimal 
algorithm is indifferent when
the expert gains are the same). By removing strictly suboptimal actions $\{\},\{12\}$ of the adversary, we obtain
\[
f(0) = (1-\delta)\left(f(-1) + \frac{1}{2}\right).
\]
Similarly, the first part of expression~\eqref{eq:Regret2} for $x <0$ boils down to
\be
f(x)=(1-\delta)\cdot\max\Big(f(x-1) + 1 - p_1,~~ f(x+1) - p_2\Big).
\label{eq:max2}
\ee
We further simplify equation~\eqref{eq:max2}. Notice that
the optimal algorithm in minimax equilibrium must make the adversary indifferent between
any two actions the adversary is randomizing over. In this case it means that for each $x< 0$ the probabilities $p_1(x)$ and $p_2(x)$ must be chosen
by the optimal algorithm in such a way that $f(x)=(1-\delta)(f(x-1) + 1 - p_1(x)) = (1-\delta)(f(x+1) - p_2(x)).$
Note that adding these two equations and dividing by 2, we get
equation~\eqref{eq:f1}. Thus we can solve for optimal regret. Additionally, solving for $p_1(x)$ and $p_2(x)$ we obtain

\begin{align}
p_1(x) = 1 - \frac{1}{2}\xi_1^x; \qquad\qquad p_2(x) = \frac{1}{2}\xi_1^x
\label{eq:pr2}
\end{align}

This proves the optimality of the algorithm~\ref{alg:2opt} for $k=2$. We summarise our
results for $k=2$ in the following theorem. For convenience we replace the negative number $x$ by positive $d
= -x$, and also replace $\xi_1$ by $\xi = \xi_2 = \frac{1}{\xi_1} \sim 1-\sqrt{2\delta}$.
\begin{oneshot}{Theorem~\ref{thm:2opt}}
In the geometric horizon model for $2$ experts with parameter $\delta \in (0,1)$:
\begin{enumerate}
\item The optimal adversary, at every time step, advances the leading expert alone with probability
$\frac{1}{2}$ and lagging expert alone with probability $\frac{1}{2}$.
\item The optimal regret is $\frac{1-\delta}{2\sqrt{1-(1-\delta)^2}} \to
\frac{1}{2}\frac{1}{\sqrt{2\delta}}$ as $\delta \to 0$.
\item The optimal algorithm, at every time step, computes the difference $d
(\geq 0)$ of cumulative gains
between the leading and lagging expert, and chooses them with probabilities
$p_1(d) = 1-\frac{1}{2}\xi^d$, and $p_2(d) = \frac{1}{2}\xi^d$. Here $\xi = \frac{1-\sqrt{1-(1-\delta)^2}}{1-\delta} \sim 1-\sqrt{2\delta}$.
\end{enumerate}
\end{oneshot}

\subsection{Two experts: interpretation as a probability matching algorithm}
\label{app:pmatching}

\paragraph{Geometric horizon model.}
The quantity $\xi$ turns out to be precisely equal to the probability that a simple random walk that starts at $1$ will reach $0$ before the geometric process gets killed. To see this, just note that it is the root of the equation which captures the probability of the above event $\xi = (1-\delta)\cdot 0 + \delta\cdot\frac{1}{2}\cdot(1+\xi^2)$ (the root that is smaller than $1$), which is $\xi = \frac{1-\sqrt{1-(1-\delta)^2}}{1-\delta}$. 
Now, note that the minimax optimal adversary advances one of the experts 
uniformly at random and doesn't advance the other. This means that the gap 
between the cumulative gains of the leading and the lagging experts evolves as 
a random walk, and the probability that given a separation of $d$, the lagging 
expert will match the leading expert is precisely $\xi^d$. Once they match, 
each expert has an equal probability $\frac{1}{2}$ of being the leading 
expert\footnote{If the experts are tied, the leader is chosen uniformly at 
random}. This means, the probability that the currently lagging expert will 
finish as the leading expert is precisely $\frac{1}{2}\xi^d$, and the 
probability that the currently leading expert will finish as the leading expert 
is $1-\frac{1}{2}\xi^d$.

\paragraph{Finite horizon model.}
We now show that for the finite horizon case too, the optimal algorithm is precisely a probability matching algorithm, i.e, the algorithm picks each expert with the probability that the respective expert finishes in the lead (we break possible ties in favor of the unique expert who doesn't have any expert ahead of him in each of the last two steps). We set up equations very similar to~\eqref{eq:Regret2} and~\eqref{eq:max2} in the finite horizon model except that $f$ now will be a function of both $x$ and the number of time steps left $\ell=T-t$ until the deadline. Thus
\begin{align}
\label{eq:prob_matching_two_exp}
f(x,0) &= 0,\qquad \mbox{if } x\le 0\\
f(x,\ell) &= \frac{f(x+1,\ell-1)+f(x-1,\ell-1)}{2},\nonumber\\
& \qquad\qquad \mbox{if } \ell > 0, \text{and } x < 0 \nonumber\\
f(0,\ell) &= f(-1,\ell-1) + \frac{1}{2},\nonumber\\
&\qquad\qquad \mbox{if } \ell > 0, x=0\nonumber.
\end{align}

We consider a simple random walk $\srw(x,\ell)$ that starts from position $x$ and does $\ell$ steps (we also use $SRW(x,\ell)$ to denote the location of this walk after $\ell$ steps). It turns out that $g(x,\ell)=\frac{\Ex{|\srw(x,\ell)|}-|x|}{2}$ satisfies exactly the same set of equations \eqref{eq:prob_matching_two_exp} as $f(x,\ell)$ does. Thus $f(x,\ell)=g(x,\ell)$. Analogously to the geometric model, we can also derive that $p_2(x,\ell) = f(x+1,\ell-1)-f(x,\ell)$ for $x<0$ and $p_2(0,\ell)=p_1(0,\ell)=\frac{1}{2}$. We immediately get the desired probability matching result for $x=0$. To get the same for $x<0$, we do a natural coupling of random walks $\srw(x+1,\ell-1)$ and $\srw(x,\ell)$ in the expression $p_2(x,\ell)=g(x+1,\ell-1)-g(x,\ell)$. When $\srw(x+1,\ell-1)$ arrives at $y$ in this coupling, $\srw(x,\ell)$ does one more iteration from the location $y-1$. The expression $g(x+1,\ell-1) - g(x,\ell)$, given that $\srw(x+1,\ell-1)$ arrives at $y$ can be written as:
\begin{align*}
&\frac{|y|-|x+1|}{2}-\frac{\frac{1}{2}|y-2|+\frac{1}{2}|y| - |x|}{2}\\
&=\frac{\frac{1}{2}(|y|-|y-2|)+1}{2}\\
&=
\begin{cases}
0 & \mbox{if } y\le 0\\
1/2 & \mbox{if } y=1 \\
1 & \mbox{if } y>1
\end{cases}
\end{align*}
The first line in the RHS of the above expression corresponds to the situations where $\srw(x,\ell)$ arrives at $y-1<0$ at step $\ell-1$, i.e., the second expert  does not reach the leader till step $\ell-1$ (and therefore the first expert is the unique one who didn't lag in steps $\ell-1$ and $\ell$); the second line corresponds to the situations where the second expert reaches the leader at step $\ell-1$ (and therefore overtakes him with probability $1/2$ in the last step); the third line represents situations when the second expert is the unique leader after $\ell-1$ steps (and therefore the second expert is the unique one that didn't lag in steps $\ell-1$ and $\ell$). This yields the desired probability matching result.

\paragraph{Uniqueness of the optimal algorithm.} In the geometric horizon 
model, we explicitly solve the infinite system of equations and realize that 
they have a unique solution proving the uniqueness of the optimal algorithm. In 
the finite horizon model, although we don't explicitly solve the system of 
equations, the discussion in the previous paragraph shows that the 
probabilities chosen by the optimal algorithm are unique, and hence the optimal 
algorithm is unique. 

\subsection{Three experts, geometric horizon: optimal algorithm, adversary and regret}
\label{app:3experts}
We derive the optimal adversary, algorithm and regret here. We restate Theorem~\ref{thm:3opt} for ease of reading.
\begin{oneshot}{Theorem~\ref{thm:3opt}}
In the geometric horizon model for $3$ experts with parameter $\delta \in (0,1)$:
\begin{enumerate}
\item The optimal regret is $\frac{2}{3}\frac{1-\delta}{\sqrt{1-(1-\delta)^2}} \to
\frac{2}{3}\frac{1}{\sqrt{2\delta}}$ as $\delta \to 0$.
\item The optimal algorithm, at every time step, computes the differences $d_{ij}$
between the cumulative gains of experts ($i$ denotes the expert with $i$th largest cumulative gains, and hence $d_{ij}\geq 0$ for all $i < j$). As a function of the $d_{ij}$'s the algorithm
follows the leading expert with probability $p_1(\mathbf{d}) = 1-\frac{\xi^{d_{12}}}{2} - \frac{\xi^{d_{13}+d_{23}}}{6}$,
the second expert with probability
$p_2(\mathbf{d}) = \frac{\xi^{d_{12}}}{2} - \frac{\xi^{d_{13}+d_{23}}}{6}$, and the
lagging expert with probability $p_3(\mathbf{d}) = \frac{\xi^{d_{13}+d_{23}}}{3}$. Here $\xi = \frac{1-\sqrt{1-(1-\delta)^2}}{1-\delta} \sim 1-\sqrt{2\delta}$.
\item The optimal adversary, at every time step, computes the differences $d_{ij}$'s, and follows the following strategies below as a function of the $d_{ij}$'s. Strategy $\{1\}\{2\}\{3\}$ means: exclusively advancing with probability $1/3$ expert $1$ (leading expert), expert $2$ (middle expert), and expert $3$ (lagging expert). Strategy $\{1\}\{23\}$ means: advancing with probability $\frac{1}{2}$ expert $1$ alone, or advancing experts $2$ and $3$ together.
\begin{description}
\item[$0 < d_{12} < d_{13}:$]  $\{1\}\{23\}$, or $\{13\}\{2\}$.
\item[$ 0 = d_{12} < d_{13}:$]  $\{1\}\{23\}$, or $\{13\}\{2\}$.
\item[$0 < d_{12} = d_{13}:$] $\{1\}\{23\}$, or $\{1\}\{2\}\{3\}$.
\item[$0 = d_{12} = d_{13}:$] $\{1\}\{2\}\{3\}$.
\end{description}
\end{enumerate}
\end{oneshot}

\paragraph{Notational convention.} At each time period $t$ we always enumerate
experts in the decreasing order of their cumulative gains $\Gainst$. We observe
that the strategy of the adversary at any moment $t$ should not change if
cumulative gains $\Gainst$ of all experts are changed by the same amount for
every expert. Thereby, at every time step $t$ we shall always adjust the total
gains $\Gainst$ of our experts, so that the leading expert $1$ has zero
cumulative gain $\Gainit[1][t]=0$. We denote the adjusted gain
$\Gainit[i+1][t]-\Gainit[1][t]$ by $x_i(t)$ for each $i\in[2]$; we denote by
$\vx(t)=(x_1(t),x_2(t))$ the vector of adjusted gains. Note that both $x_1(t)$
and $x_2(t)$ are negative.

We denote by $f(\vx)$ the optimal regret the adversary can obtain for an
initial configuration $\Gains = (0,x_1,x_2)$ (where again $x_1, x_2 \leq 0$).
Much like the case of $k=2$, the advantage of this convention is that if we start
from configuration $(0,x_1,x_2)$ and stop immediately, the ``max-expert-gain -
algorithm's gain'' is just $0$.

Algorithm assigns probabilities $p_1(\vx)$, $p_2(\vx)$, and $p_3(\vx) =
1-p_1(\vx) - p_2(\vx)$ respectively to the leading, middle and lagging experts.
Similarly to the case $k=2$ the adversary now has eight choices and the regret
satisfies the following expression for each $\vx: 0 > x_1 > x_2$.

\begin{strip}
\be
f(x_1, x_2) =\delta\cdot 0 + (1-\delta)\cdot\max
\begin{cases}
f(x_1 - 1, x_2 - 1) +   1 - p_1     & \quad //~~~\{1\}   \\
f(x_1 + 1, x_2 + 1) - p_2 - p_3     &  \quad //~~~\{23\}  \\
f(x_1 - 1, x_2) + 1 - p_1 - p_3     &  \quad //~~~\{13\}  \\
f(x_1 + 1, x_2)     - p_2           &  \quad //~~~\{2\}   \\
f(x_1, x_2 - 1) + 1 - p_1 - p_2     &  \quad //~~~\{12\}  \\
f(x_1, x_2 + 1)     - p_3           &  \quad //~~~\{3\}   \\
f(x_1, x_2) + 1 - p_1 - p_2 - p_3   &  \quad //~~~\{123\} \\
f(x_1, x_2)                         &  \quad //~~~\{\}
\end{cases}
\label{eq:meaty_regret3}
\ee
We note that we can omit the lines $\{123\}$ and $\{\}$ in the RHS of the
expression above.  For the boundary points $\vx: 0 = x_1 > x_2$ and $\vx: 0> x=
x_1 = x_2$ we need to take into account in \eqref{eq:meaty_regret3} the
possibility that the order $0\ge x_1\ge x_2$ might  change:

\begin{align}
\frac{f(0, x_2)}{1-\delta}=\max
\begin{cases}
f(-1, x_2 - 1) + 1 - p_1              \\
f(-1, x_2)     + 1 - p_2 - p_3        \\
f(-1, x_2)     + 1 - p_1 - p_3        \\
f(-1, x_2 - 1) + 1 - p_2              \\
f( 0, x_2 - 1) + 1 - p_1 - p_2        \\
f( 0, x_2 + 1)     - p_3
\end{cases}
& \hspace{-1em}
\frac{f(x, x)}{1-\delta}=\max
\begin{cases}
f(x - 1, x - 1) + 1 - p_1          &   //\{1\}\\
f(x + 1, x + 1)     - p_2 - p_3    &   //\{23\}\\
f(x    , x - 1) + 1 - p_1 - p_3    &   //\{13\}\\
f(x + 1, x)         - p_2          &   //\{2\}\\
f(x    , x - 1) + 1 - p_1 - p_2    &   //\{12\}\\
f(x + 1, x    )     - p_3          &   //\{3\}
\end{cases}
\nonumber \\
& \label{eq:boundary_regret3}
\end{align}
\end{strip}

For $\vx: 0 = x_1 = x_2$ we have

\begin{strip}
\be
f(0, 0)=(1-\delta)\max
\begin{cases}
f(-1, -1) + 1 - p_1          &  ~~~ //\{1\}   \\
f(0 , -1) + 1 - p_2 - p_3    &  ~~~ //\{23\}  \\
f(0 , -1) + 1 - p_1 - p_3    &  ~~~ //\{13\}  \\
f(-1, -1) + 1 - p_2          &  ~~~ //\{2\}   \\
f(0 , -1) + 1 - p_1 - p_2    &  ~~~ //\{12\}  \\
f(-1, -1) + 1 - p_3          &  ~~~ //\{3\}
\end{cases}
\label{eq:zero_regret3}
\ee
\end{strip}

Our approach here will be a guess and verify approach. While there are several
strategies possible for the adversary, we discovered that the optimal strategy for the adversary is to play $\{1\},\{23\}$
or $\{13\},\{2\}$ for most of the $\vx$ (at least for those $\vx: 0 > x_1 >
x_2$). We will now compute the consequences of this being the optimal adversary
and finally verify if our guess was true.  So playing  $\{1\},\{23\}$ or
$\{13\},\{2\}$ for most of the time means that for $\vx: 0 > x_1 > x_2$ we
have

\begin{align}
f(x_1,x_2) &= \frac{1-\delta}{2}\Big[f(x_1 - 1,x_2) + f(x_1 + 1, x_2)\Big]\nonumber\\
f(x_1,x_2) &= \frac{1-\delta}{2}\Big[f(x_1 - 1,x_2 - 1) +\nonumber\\
&\qquad\qquad\qquad f(x_1 + 1, x_2 + 1)\Big].
\label{eq:meaty}
\end{align}

One can write generating function for $f(x_1,x_2):$

\[
G(u,v)=\sum_{x_1,x_2} f(x_1,x_2)u^{x_1}v^{x_2}
\]

We can write two functional relations on $G(u,v)$ from expression
\eqref{eq:meaty} and further derive a parametric expression for $f(x_1,x_2)$:

\[
f(x_1,x_2)= c_1\cdot\xi_1^{x_1} + c_2\cdot\xi_1^{2x_2-x_1} + c_3\cdot\xi_1^{-x_1} +
c_4\cdot\xi_1^{x_1-2x_2},
\]
where $c_1,c_2,c_3,c_4$ are unknown parameters and $\xi_1 > 1 >\xi_2$ are the
roots of the characteristic polynomial $x^2-\frac{2}{1-\delta}x+1=0$. It turns
out that, as the regret cannot grow faster than a linear function and cannot be
negative, it follows that $c_4$ and also $c_3$ must be 0.

From the algorithm's  point of view, the probabilities $p_1(\vx)$, $p_2(\vx)$, and
$p_3(\vx)$ must be chosen in such a way that adversary will be indifferent
between playing $\{1\}$, $\{23\}$, $\{13\}$, and $\{2\}$ for $\vx: 0 > x_1 >
x_2$. From this condition we derive that

\begin{align*}
1 - p_1 &= \left(\frac{\xi_1-\xi_2}{2}\right)\left(c_1\cdot\xi_1^{x_1} + c_2\cdot\xi_1^{2x_2-x_1}
         \right)\\
p_2 &= \left(\frac{\xi_1-\xi_2}{2}\right)\left(c_1\cdot\xi_1^{x_1} - c_2\cdot\xi_1^{2x_2-x_1}
         \right)\\
p_3 &= \left(\frac{\xi_1-\xi_2}{2}\right)\left(2c_2\cdot\xi_1^{2x_2-x_1}
         \right)
\end{align*}
We also assume that the above formula for $\vp(\vx)$ extends to the points of
the form $\vx: 0 = x_1> x_2$ and $\vx: 0>x_1=x_2$ and $\vx: 0 = x_1 =x_2.$ We
equate $p_1(\vx)$ and $p_2(\vx)$ for $\vx: 0 = x_1> x_2$, as now leading and
middle experts are identical from the adversary's point of view. Similarly, we
equate $p_2(\vx)$ and $p_3(\vx)$ for $\vx: 0>x_1=x_2$; and equate $p_1(\vx)$,
$p_2(\vx)$ and $p_3(\vx)$ for $\vx: 0 = x_1 =x_2.$ From these equations we
deduce that
\begin{align*}
c_1 = \frac{1}{\xi_1-\xi_2}; \qquad
c_2 = \frac{1}{3\left(\xi_1-\xi_2\right)}.
\end{align*}
This results in the following expression for the regret $f(\vx)$:
\be
f(\vx)= \frac{\xi_1^{x_1}}{\xi_1-\xi_2} + \frac{\xi_1^{2x_2-x_1}}{3\left(\xi_1-\xi_2\right)},
\label{eq:regret3}
\ee
which gives us regret of $\frac{4}{3\left(\xi_1-\xi_2\right)}$ at $\vx=(0,0)$. As $\delta\to 0$ the regret

\[
\regretd=\frac{4}{3\left(\xi_1-\xi_2\right)}=\frac{2(1-\delta)}{3\sqrt{\delta\cdot(2-\delta)}}
\underset{_{\delta\to 0}}{\longrightarrow} \frac{2}{3\sqrt{2\delta}}.
\]

\begin{theorem} Equation \eqref{eq:regret3} gives the precise value of the
regret for every normalized $\vx:0\ge x_1\ge x_2$. Moreover, the optimal
algorithm chooses leading, middle and lagging experts respectively with the
following probabilities $p_1(\vx)$, $p_2(\vx)$, and $p_3(\vx)$:
\begin{align}
1 - p_1(\vx) &= \frac{\xi_1^{x_1}}{2} + \frac{\xi_1^{2x_2-x_1}}{6}  \nonumber    \\
p_2(\vx)     &= \frac{\xi_1^{x_1}}{2} - \frac{\xi_1^{2x_2-x_1}}{6}  \nonumber    \\
p_3(\vx)     &= \frac{\xi_1^{2x_2-x_1}}{3}
\label{eq:alg_prob3}
\end{align}
\label{th:correctness_regret}
\end{theorem}

\begin{proof} To prove this theorem we shall first verify that the function $f(\cdot)$ given
by \eqref{eq:regret3} together with the probabilities \eqref{eq:alg_prob3}
satisfies combined system of equations \eqref{eq:meaty_regret3},
\eqref{eq:boundary_regret3}, \eqref{eq:zero_regret3} for every $\vx:0\ge x_1\ge
x_2$.

Then the expression \eqref{eq:regret3} immediately gives us an upper bound on
the regret function $f(\vx)$. Indeed, if we fix strategy of the algorithm to be
as in \eqref{eq:alg_prob3}, then $f(\vx)$ would be an upper bound on the regret
that the best response adversary (with respect to this fixed algorithm) could
get.

Finally, to show matching lower bound we will consider the best response
strategy of the adversary in \eqref{eq:regret3}, i.e., those lines in RHS of
\eqref{eq:meaty_regret3}, \eqref{eq:boundary_regret3}, \eqref{eq:zero_regret3} which
are equal to LHS. We will make sure that among these strategies the adversary
can always compose a mixed strategy which is balanced, i.e. the one that makes
algorithm completely indifferent between all experts. Assume that we have
restricted our adversary to these mixed strategies. 
Then any algorithm will be
the best response algorithm, in particular the algorithm defined by
\eqref{eq:alg_prob3}. Hence, this particular restricted strategy of the
adversary provides a lower bound given by \eqref{eq:regret3} on the regret
function $f(\vx)$.


We begin by verifying \eqref{eq:meaty_regret3} for the interior points $\vx:0> x_1> x_2.$
\begin{lemma}
Equation \eqref{eq:meaty_regret3} holds true for the interior points $\vx:0> x_1> x_2.$
\label{le:meaty}
\end{lemma}
\begin{proof}
\begin{strip}
\bee
\frac{\xi_1^{x_1}}{\xi_1-\xi_2} + \frac{\xi_1^{2x_2-x_1}}{3\left(\xi_1-\xi_2\right)}= (1-\delta)\cdot\max
\begin{cases}
\frac{\xi_1^{x_1-1}}{\xi_1-\xi_2} + \frac{\xi_1^{2x_2-x_1-1}}{3\left(\xi_1-\xi_2\right)}
 +   \frac{\xi_1^{x_1}}{2} + \frac{\xi_1^{2x_2-x_1}}{6}
  &  \quad //~~~\{1\}   \\
\frac{\xi_1^{x_1+1}}{\xi_1-\xi_2} + \frac{\xi_1^{2x_2-x_1+1}}{3\left(\xi_1-\xi_2\right)}
 -   \frac{\xi_1^{x_1}}{2} - \frac{\xi_1^{2x_2-x_1}}{6}
  &  \quad //~~~\{23\}  \\
\frac{\xi_1^{x_1-1}}{\xi_1-\xi_2} + \frac{\xi_1^{2x_2-x_1+1}}{3\left(\xi_1-\xi_2\right)}
 +   \frac{\xi_1^{x_1}}{2} - \frac{\xi_1^{2x_2-x_1}}{6}
  &  \quad //~~~\{13\}  \\
\frac{\xi_1^{x_1+1}}{\xi_1-\xi_2} + \frac{\xi_1^{2x_2-x_1-1}}{3\left(\xi_1-\xi_2\right)}
 -   \frac{\xi_1^{x_1}}{2} + \frac{\xi_1^{2x_2-x_1}}{6}
  &  \quad //~~~\{2\}   \\
\frac{\xi_1^{x_1}}{\xi_1-\xi_2} + \frac{\xi_1^{2x_2-x_1-2}}{3\left(\xi_1-\xi_2\right)}
 +   \frac{\xi_1^{2x_2-x_1}}{3}
  &  \quad //~~~\{12\}  \\
\frac{\xi_1^{x_1}}{\xi_1-\xi_2} + \frac{\xi_1^{2x_2-x_1+2}}{3\left(\xi_1-\xi_2\right)}
 -   \frac{\xi_1^{2x_2-x_1}}{3}
  &  \quad //~~~\{3\}   \\
\frac{\xi_1^{x_1}}{\xi_1-\xi_2} + \frac{\xi_1^{2x_2-x_1}}{3\left(\xi_1-\xi_2\right)}
  &  \quad //~~~\{123\} \\
\frac{\xi_1^{x_1}}{\xi_1-\xi_2} + \frac{\xi_1^{2x_2-x_1}}{3\left(\xi_1-\xi_2\right)}
  &  \quad //~~~\{\}
\end{cases}
\eee
\end{strip}
Clearly, the lines $\{\}$ and $\{123\}$ in the RHS are smaller than the LHS.
Since we have chosen $f(\vx)$ according to \eqref{eq:meaty}, it immediately
follows that the average of the lines $\{1\}$ and $\{23\}$ in RHS as well as
average of the lines $\{13\}$ and $\{2\}$ in RHS are equal to the the LHS. We
further notice that $p_1, p_2, p_3$ were chosen so that the RHS expressions in
lines $\{1\}$ and $\{23\}$ are equal as well as are equal expressions in lines
$\{13\}$ and $\{2\}$. This makes the expressions in lines 1-4 in the RHS to be
equal to the LHS.

We are only left to verify that LHS is greater than or equal to the expressions
in the lines 5 and 6 in the RHS. We recall that $\xi_1>1>\xi_2$ are the roots
of the polynomial $x^2-\frac{2}{1-\delta}x+1$, so that $\xi_1\cdot\xi_2=1$ and
$\xi_1 + \xi_2 = \frac{2}{1-\delta}$.

For the line \{12\} in RHS we need to verify the following.

\begin{strip}
\bee
\frac{1}{1-\delta}\left(\frac{\xi_1^{x_1}}{\xi_1-\xi_2} +
\frac{\xi_1^{2x_2-x_1}}{3\left(\xi_1-\xi_2\right)}\right)
\ge
\frac{\xi_1^{x_1}}{\xi_1-\xi_2} + \frac{\xi_1^{2x_2-x_1-2}}{3\left(\xi_1-\xi_2\right)}  +   \frac{\xi_1^{2x_2-x_1}}{3}
\eee

Equivalently, we need to show

\begin{align*}
\left(\frac{\delta}{1-\delta}\right)\frac{\xi_1^{x_1}}{\xi_1-\xi_2} 
& \ge \frac{\xi_1^{2x_2-x_1}}{3\left(\xi_1-\xi_2\right)} \left(\xi_2^2-\frac{1}{1-\delta}+\xi_1-\xi_2\right)
\\
&= \frac{\xi_1^{2x_2-x_1}}{3\left(\xi_1-\xi_2\right)}
\left(\xi_2^2-\frac{1}{1-\delta}+\frac{2}{1-\delta}-2\xi_2\right)\\
&= \frac{\xi_1^{2x_2-x_1}}{3\left(\xi_1-\xi_2\right)}
\left(\frac{2}{1-\delta}\xi_2-1-2\xi_2+\frac{1}{1-\delta}\right)\\
&= \frac{\xi_1^{2x_2-x_1}}{3\left(\xi_1-\xi_2\right)}
\left(\frac{\delta}{1-\delta}\right)\left(2\xi_2+1\right).
\end{align*}

The last inequality holds true as
\bee
\xi_1^{2x_1-2x_2}\ge 1 \ge  \frac{2\xi_2+1}{3}.
\eee

Similarly for the line $\{3\}$ in RHS we need to verify that
\bee
\frac{1}{1-\delta}\left(\frac{\xi_1^{x_1}}{\xi_1-\xi_2} +
\frac{\xi_1^{2x_2-x_1}}{3\left(\xi_1-\xi_2\right)}\right) \ge
\frac{\xi_1^{x_1}}{\xi_1-\xi_2} + \frac{\xi_1^{2x_2-x_1+2}}{3\left(\xi_1-\xi_2\right)}
 -   \frac{\xi_1^{2x_2-x_1}}{3}
\eee

After some transformation we need to show that

\begin{align*}
\left(\frac{\delta}{1-\delta}\right)\frac{\xi_1^{x_1}}{\xi_1-\xi_2} & \ge
\frac{\xi_1^{2x_2-x_1}}{3\left(\xi_1-\xi_2\right)} \left(\xi_1^2-\frac{1}{1-\delta}-\xi_1+\xi_2\right)\\
&= \frac{\xi_1^{2x_2-x_1}}{3\left(\xi_1-\xi_2\right)}
\left(\xi_1^2-\frac{1}{1-\delta}+\frac{2}{1-\delta}-2\xi_1\right) \\
&= \frac{\xi_1^{2x_2-x_1}}{3\left(\xi_1-\xi_2\right)}
\left(\frac{2}{1-\delta}\xi_1-1-2\xi_1+\frac{1}{1-\delta}\right)\\
&= \frac{\xi_1^{2x_2-x_1}}{3\left(\xi_1-\xi_2\right)}
\left(\frac{\delta}{1-\delta}\right)\left(2\xi_1+1\right).
\end{align*}
\end{strip}

We further compare LHS with RHS of the last inequality. We need to prove that

\bee
\xi_1^{2x_1-2x_2} \ge  \frac{2\xi_1+1}{3}.
\eee

Since $x_1>x_2$, we observe that $\xi_1^{2x_1-2x_2}\ge\xi_1^2$. The desired
inequality is true, as $\xi_1^2\ge\frac{2\xi_1+1}{3}.$
\end{proof}

We next consider a few cases for boundary points when there are ties between leading, middle and legging experts. However, if there is no change in the order of experts after adversary's action, most of our derivations in Lemma~\ref{le:meaty} applies to the boundary cases as well.

\begin{lemma}
Equation \eqref{eq:boundary_regret3} holds true for the boundary points $\vx:0= x_1> x_2.$
\label{le:boundary00x}
\end{lemma}
\begin{proof}
We note that the expressions in the lines $\{1\}$, $\{13\}$, $\{12\}$, and $\{3\}$  of \eqref{eq:boundary_regret3} are the same as in Lemma~\ref{le:meaty} for $x_1 = 0$, since the order of leading, middle, and legging experts does not change for any of these choices of the adversary.

We observe that $p_1 = 1 - \frac{1}{2} - \frac{\xi_1^{2x_2}}{6}= \frac{1}{2} -
\frac{\xi_1^{2x_2}}{6} = p_2$ for boundary points $\vx: 0= x_1> x_2$.
Furthermore, as first two experts are the same, leading and middle expert are equivalent from the perspective of the adversary. It implies that lines $\{1\}$ and $\{2\}$ as well as lines $\{13\}$ and $\{23\}$ in the RHS of \eqref{eq:boundary_regret3} are identical. We conclude the proof by observing that
\begin{enumerate}
\item LHS is equal to the line $\{1\}$ in RHS (same argument as in Lemma~\ref{le:meaty}), which is equal to the expression in the line $\{2\}$ of RHS.
\item LHS is equal to the expression in the line $\{13\}$ in RHS (same argument as in Lemma~\ref{le:meaty}), which is the same as the line $\{23\}$ in RHS.
\item LHS is at least the expressions in the lines $\{12\}$ and $\{3\}$ (same argument as in Lemma~\ref{le:meaty}). Indeed, we only used the fact that $x_1-x_2 > 0$ and analytically all the rest derivations remain the same as in Lemma~\ref{le:meaty}.
\end{enumerate}
\end{proof}

\begin{lemma}
Equation \eqref{eq:boundary_regret3} holds true for the boundary points $\vx:0 > x = x_1 = x_2.$
\label{le:boundary0xx}
\end{lemma}
\begin{proof}
We note that the expressions in the lines $\{1\}$, $\{2\}$, $\{12\}$, and $\{23\}$  of \eqref{eq:boundary_regret3} are the same as in Lemma~\ref{le:meaty} for $x = x_1 = x_2$, since the order of leading, middle, and legging experts does not change for any of these choices of the adversary.

We observe that $p_2 = \frac{\xi_1^{x}}{2} - \frac{\xi_1^{x}}{6}  = \frac{\xi_1^{x}}{3} = p_3$
for boundary points $\vx: 0> x =  x_1 = x_2$. Furthermore, as last two experts are the same, middle and legging experts are equivalent from the perspective of the adversary. It implies that lines $\{2\}$ and $\{3\}$ as well as lines $\{13\}$ and $\{12\}$ in the RHS of \eqref{eq:boundary_regret3} are identical. We conclude the proof by observing that
\begin{enumerate}
\item LHS is equal to the expression in the line $\{2\}$ in RHS (same argument as in Lemma~\ref{le:meaty}), which is equal to the expression in the line $\{3\}$ of RHS.
\item LHS is at least the expression in the line $\{12\}$ in RHS (same argument as in Lemma~\ref{le:meaty}), which is the same as the line $\{13\}$ in RHS. Indeed, for the line $\{12\}$ we don't need $x_1$ to be strictly greater than $x_2$ and our derivations as in Lemma~\ref{le:meaty} do not change.
\item LHS is equal to the expressions in the lines $\{12\}$ and $\{3\}$ (same argument as in Lemma~\ref{le:meaty}).
\end{enumerate}
\end{proof}

\begin{lemma}
Equation \eqref{eq:zero_regret3} holds true for the boundary point $\vx:0 = x_1 = x_2.$
\label{le:vertex3}
\end{lemma}
\begin{proof}
We observe that $p_1 = 1- \frac{\xi_1^{0}}{2} - \frac{\xi_1^{0}}{6} =  \frac{\xi_1^{0}}{2} - \frac{\xi_1^{0}}{6} = p_2  = \frac{\xi_1^{0}}{3} = p_3$. Therefore, the lines $\{1\}$, $\{2\}$, and $\{3\}$ in RHS are identical, similarly are identical the lines $\{12\}$, $\{23\}$, and $\{13\}$. We also notice that expression in the lines $\{1\}$ and $\{12\}$ in \eqref{eq:zero_regret3} are special cases of the corresponding expressions in \eqref{eq:meaty_regret3} for $x_1=x_2=0$. Expression in the line $\{12\}$ in RHS is not greater than LHS, because our derivations from Lemma~\ref{le:meaty} analytically remain the same and for the expression in line $\{12\}$ we only need $x_1\ge x_2$.

We conclude the proof by observing that
\begin{enumerate}
\item LHS is equal to the expression in the line $\{1\}$ in RHS (same argument as in Lemma~\ref{le:meaty}), which is equal to the expressions in the lines $\{2\}$ and $\{3\}$ in RHS.
\item LHS is at least the expression in the line $\{12\}$ in RHS (same argument as in Lemma~\ref{le:meaty}), which is the same as the lines $\{13\}$ and $\{23\}$ in RHS.
\end{enumerate}
\end{proof}

We summarize below the best choices for the adversary (lines in RHS of \eqref{eq:meaty_regret3},\eqref{eq:boundary_regret3},\eqref{eq:zero_regret3} which are equal to LHS).
\begin{description}
\item[$\vx: 0 > x_1 > x_2$]  $\{1\}$, $\{23\}$, $\{13\}$, $\{2\}$.
\item[$\vx: 0 = x_1 > x_2$]  $\{1\}$, $\{2\}$, $\{13\}$, $\{23\}$.
\item[$\vx: 0 > x_1 = x = x_2$] $\{1\}$, $\{23\}$, $\{2\}$, $\{3\}$.
\item[$\vx: 0 = x_1 = x_2$] $\{1\}$, $\{2\}$, $\{3\}$.
\end{description}

The corresponding mixed balanced strategies of the adversary are:
\begin{description}
\item[$\vx: 0 > x_1 > x_2$]  $\{1\}\{23\}$, or $\{13\}\{2\}$.
\item[$\vx: 0 = x_1 > x_2$]  $\{1\}\{23\}$, or $\{13\}\{2\}$.
\item[$\vx: 0 > x_1 = x = x_2$] $\{1\}\{23\}$, or $\{1\}\{2\}\{3\}$.
\item[$\vx: 0 = x_1 = x_2$] $\{1\}\{2\}\{3\}$.
\end{description}
This concludes the proof of Theorem~\ref{th:correctness_regret} and hence
Theorem~\ref{thm:3opt}.
\end{proof}

\section{Comparison with multiplicative weights algorithm}
\label{app:mwa}
In this section, we show that the optimal algorithm is not in the family of multiplicative weight algorithms.

\paragraph{Multiplicative weights algorithm (MWA).} Given cumulative gains
$\Gainit[1][t-1],\dots,\Gainit[k][t-1]$ for the $k$ experts after $t-1$ steps,
MWA computes the exponentials of these cumulative gains and follows expert $i$
with probability proportional to these exponentials. Formally, MWA at time
$t$ follows expert $i$ with probability $\frac{\exp(\eta\Gainit[i][t-1])}{\sum_j
\exp(\eta\Gainit[j][t-1])}$, where $\eta$ is a parameter that can be tuned. For
the special case of $2$ experts, this description can be simplified: let
$d(t-1) = \Gainit[1][t-1]-\Gainit[2][t-1]$ where we use $1$ and $2$ denote the
leading and lagging experts respectively. Then, MWA follows the leading expert
with probability $\frac{e^{\eta d(t-1)}}{e^{\eta d(t-1)}+1}$ and the lagging expert with probability $\frac{1}{e^{\eta d(t-1)}+1}$.

\subsection{Optimal algorithm is not in the MWA family}
Even for $k=2$ experts, the optimal algorithm is not in the MWA family. From the tuple representation of MWA and OPT, namely, 
\[
\text{MWA:$\left(\frac{e^{\eta
d}}{e^{\eta d}+1}, \frac{1}{e^{\eta d}+1}\right)$, and OPT:
$\left(1-\frac{1}{2}\xi^d, \frac{1}{2}\xi^d\right)$,} 
\]
it is clear that the
optimal algorithm cannot be expressed as a multiplicative weights algorithm. We
now show that even a convex combination of MWAs cannot express it.

\begin{fact}
\label{fact:MWAconvex}
No convex combination of multiplicative weight algorithms can express the 
optimal algorithm.
\end{fact}
\begin{proof}
We show that even if at every step, the parameter $\eta$ was allowed to
be drawn from a measure $\mu$, MWA cannot express the optimal algorithm, i.e.,
for no measure $\mu$ can we have that for all integer $d \geq 0$,
$\int_{-\infty}^{\infty} \frac{\mathrm{d}\mu(\eta)}{e^{\eta d}+1} = \frac{1}{2}\xi^d$,
or equivalently, $\int_{-\infty}^{\infty} \frac{\mathrm{d}\mu(\eta)}{(\xi e^{\eta})^
d+\xi^d} = \frac{1}{2}$. Let $\eta_0$ be such that $\xi e^{\eta_0}=1$ (note
that $\xi < 1$). The measure on $\{\eta: \eta < \eta_{0}\}$ should be $0$ for \
otherwise the denominator in the integral goes to $0$ as $d\to\infty$, which
will make the integral go to $\infty$ where as the RHS is just $\frac{1}{2}$.
Likewise, any measure on ${\eta: \eta > \eta_{0}}$ doesn't contribute to the
integral as $d\to\infty$ since the integral will anyway be $0$ in the region
$(\eta_0,\infty)$. Thus, for the integral to be $\frac{1}{2}$ as $d\to\infty$,
we need $\mu(\eta_0) = \frac{1}{2}$.  We now expand twice the LHS, namely, the
integral  $\int_{\eta_0}^{\infty} \frac{2 \mathrm{d}\mu(\eta)}{(\xi e^{\eta})^ d+\xi^d}$ by splitting it into two terms: the first term is the integral at $\eta_0$ where the measure is $\mu(\eta_0) = \frac{1}{2}$ and $\xi e^{\eta_0} = 1$, and the second term is the integral in the region $(\eta_0,\infty)$. So we have $\int_{\eta_0}^{\infty} \frac{2 \mathrm{d}\mu(\eta)}{(\xi e^{\eta})^ d+\xi^d} = \frac{1}{\xi^d+1}+\int_{(\eta_0,\infty)}\frac{2 \mathrm{d}\mu(\eta)}{(\xi e^{\eta})^d+\xi^d}.$ If this were to be equal to twice the RHS, namely $1$, we need $\frac{\xi^d}{\xi^d+1} = \int_{(\eta_0,\infty)}\frac{2 \mathrm{d}\mu(\eta)}{(\xi e^{\eta})^
d+\xi^d}$. Or equivalently, we need $\frac{1}{\xi^d+1} = \int_{(\eta_0,\infty)}\frac{2 \mathrm{d}\mu(\eta)}{(\xi^2 e^{\eta})^
d+\xi^{2d}}$. Now, as $d\to\infty$, the LHS approaches $1$. For the RHS to approach $1$ as $d\to\infty$, we need that $\mu(\eta')=\frac{1}{2}$, where $\eta'$ is such that $\xi^2\eta'=1$. This completes the proof.
\end{proof}

\end{document}